\documentclass[11pt,hidelinks]{article}

\usepackage{fullpage}

\usepackage{amsmath,amssymb,amsthm, amsfonts}

\usepackage{enumitem, comment, xifthen}

\usepackage{hyperref}       
\usepackage{url}            
\usepackage{booktabs}       
\usepackage{amsfonts}       
\usepackage{nicefrac}       

\usepackage{algorithm}

\usepackage{macros}
\def\Alg{\textsc{Alg}}
\def\wubar{\underline w}

\def\wstar{w^\star}
\def\what{\widehat w}

\usepackage{centernot}
\usepackage{amssymb}
\usepackage{comment}
\usepackage{enumitem}
\usepackage{wrapfig}
\usepackage{algorithm}
\usepackage{algpseudocode}
\usepackage{cleveref}

\begin{document}

\begin{center}

  {\bf{\LARGE Provable Benefits of Actor-Critic Methods \\ for Offline
      Reinforcement Learning}} \\

  \vspace*{.2in}

  \begin{tabular}{ccc}
    Andrea Zanette && Martin J. Wainwright \\
    Inst. for Comp. and Math. Engineering &&
    Departments of Statistics and EECS \\
\texttt{zanette@stanford.edu} && \texttt{wainwrig@berkeley.edu}
  \end{tabular}
\begin{tabular}{c}  
  Emma Brunskill \\
  Department of Computer Science \\
  Stanford University \\
\texttt{ebrun@cs.stanford.edu} 
\end{tabular}

  \vspace*{.2in}

  \begin{abstract}
    Actor-critic methods are widely used in offline reinforcement learning
practice, but are not so well-understood theoretically. We propose a new
offline actor-critic algorithm that naturally incorporates the pessimism principle, leading to several key advantages compared to the state of the art. 
The algorithm can operate when the Bellman evaluation operator is closed with respect to the action value function of the actor's policies; this is a more general setting than the low-rank MDP model. Despite the added generality, the procedure is computationally tractable as it involves the solution of a sequence of second-order programs.
We prove an upper bound on the suboptimality gap of the policy returned by the procedure that depends on the data coverage of any arbitrary, possibly data dependent comparator policy.
The achievable guarantee is complemented with a minimax lower bound that is matching up to logarithmic factors.
  \end{abstract}
\end{center}


\section{Introduction}

The problem of learning a near-optimal policy is a core challenge in
reinforcement learning (RL).  In many settings, it is beneficial to be able to  learn a good policy using only a pre-collected set of data, 
without further exploration with the environment;
this problem is known as \emph{offline or batch policy learning}.  The offline
setting has unique challenges due to the incomplete information about
the Markov decision process (MDP) encoded in the available
dataset. For example, due to maximization bias, a naive offline
algorithm can return a policy with a severely overestimated value.  In order to avoid such 
undesirable behavior, researchers have introduced the idea of
pessimism under uncertainty, and there is now a growing literature
(e.g.,~\cite{liu2020provably,jin2020pessimism,buckman2020importance,kumar2019stabilizing,kidambi2020morel,yu2020mopo})
on different ways in which pessimism can be incorporated.  See
Appendix~\ref{sec:Literature} for additional references and
discussion of this body of work.

At a high level, incorporating pessimism prevents algorithms from
settling down on uncertain policies whose value might be misleadingly
high under the current dataset due to statistical errors. By using
pessimism, uncertain policies are penalized in such a way that only
those policies robust to statistical errors are returned. The
principle can be implemented in at least two different ways: (a) by
penalizing policies that are far from the one that generated the
dataset; or (b) by penalizing the value functions of policies not well
covered by the dataset. In this paper, we take the latter avenue.

\subsection{Overview and our contributions}

Implementing pessimism with function approximation is challenging for
several reasons.  First, uncertainty must be estimated with particular
care.  On one hand, underestimating it can fail to correct the
coverage problem.  On the other hand, overestimating it leads to
policies that are too conservative and thus underperform. Second, the
incorporation of pessimism may introduce complex, higher order
perturbations into the value function class handled by the algorithm.
Similar issues can arise when adding optimistic bonuses in the
exploration. The increased complexity of the function class often
requires additional assumptions on the model, because the new class
needs to interact ``nicely'' with the Bellman operator.  Prior art on
pessimism with function approximation has by-passed this problem by
making strong model assumptions, such as low-rank
transitions~\cite{jin2020pessimism} or algorithm-specific
assumptions~\cite{liu2020provably}.

\paragraph{Actor-critic methods:}
Most past theoretical work on offline reinforcement learning on finding with high probability the policy with the highest performance has focused on algorithms that
are either model or
value-based\footnote{Exceptions to this include importance-sampling based approaches to selecting among a finite set of policies (e.g.~\cite{mandel2014offline,thomas2015high,thomas2019preventing}); however, such approaches have focused on operating without a Markov assumption and inherently provide much looser guarantees than the ones we and others consider for the Markov setting.}
~\cite{liu2020provably,jin2020pessimism,buckman2020importance,kidambi2020morel,yu2020mopo}; these often incorporate pessimism into the estimates of the policy performance.
Actor-critic methods are a hybrid class of methods that mitigate some
deficiencies of methods that are either purely policy or purely
value-based~\cite{konda2000actor,KonTsi03,heess2015learning,haarnoja2017reinforcement,haarnoja2018soft};  in modern RL, they are
widely used in practice
(e.g.,~\cite{levine2020offline,wu2019behavior,wu2021uncertainty,kumar2019stabilizing,kumar2020conservative}).
An actor-critic method generally consists of an actor that changes the
policy in order to maximize its value as estimated by the critic.
Given their popularity, it is natural to ask the following question:
\emph{do actor-critic methods provably offer any advantage in offline
RL?}  The main contribution of this paper is to give a positive answer
to this question: by separating the policy optimization from the
policy evaluation, both tasks become simpler to design and the
pessimism principle can be incorporated more naturally.

\paragraph{Contributions:}
More specifically, we study the problem of policy learning using
linear function approximation in the offline setting.  We assume that
we are given a batch data set $\D$, in which each sample consists of a
quadruple.  The first two components are the state-action pair, corresponding
to the state in which a given action was taken, and the last two
components correspond to a noisy observation of the reward, and a
successor state drawn from the appropriate transition function.  Our
theory allows for a very general dependence structure among the the
state-action pairs in these samples; when the data set is ordered
according to how the samples were collected (which need not be related
to a trajectory), we allow the state-action pair at any given instant
to depend on all past samples.  This set-up allows from data collected
from arbitrary policies, mixtures of policies, generative models or even in adversarial
manner.

Given such a data set, our objective is to find the policy that
performs best in the face of uncertainty.  In particular, we need to
account for the fact that the optimal policy $\pi^*$ for the
underlying MDP may not be well covered by the dataset $\D$, in which
case the associated uncertainty would be prohibitive.  In order to
achieve this goal, we design an actor-critic procedure that
iteratively optimizes a lower bound on the value of the optimal
policy.  Suppose that we are interested in optimizing the value
function at some given initial $\state_1$.  Our strategy works as
follows: for any given policy $\pi$, we construct a family
$\MDPClass(\pi)$ of ``statistically plausible'' MDPs, and use them to
define a simple second-order cone program.  By solving this convex
program, we obtain value function estimate $\Vubar^\pi_M(\state_1) =
\arg \min_{\MDP \in \MDPClass(\pi)} V^\pi_\MDP(\state_1)$ that---for
an appropriately constructed family $\MDPClass(\pi)$---is guaranteed
to be a lower bound on the true value function of $\pi$ in the unknown
MDP that generated the dataset.  Given a procedure for producing such
lower bounds, it is then natural to maximize these lower bounds over
some family $\Pi$ of policies.  This combination leads to the
saddle-point problem
\begin{align}
\label{eqn:minimax}
\max_{\pi \in \Pi} \min_{\MDP \in \MDPClass(\pi)}
V^{\pi}_\MDP(\state_1).
\end{align}  
Note that actor-critic methods fit naturally in this framework: the
critic provides a pessimistic evaluation of any given policy $\pi$,
and the actor solves the outer maximization problem over policies.
This decoupling lends itself to a computationally tractable
implementation, along with an analysis of the procedure.  In
particular, we show that the actor's sequence of estimated policies
enjoys online learning-style guarantees with respect to a sequence of
pessimistic MDPs implicitly identified by the critic.

The way in which we introduce pessimism is a second key component of
the algorithmic framework.  In particular, in line with our previous
paper~\cite{zanette2020learning}, we do so without enlarging the
prescribed classes of functions and policies.  We do so by a direct
perturbation of the value functions examined by the critic; there is
no addition of pessimistic bonuses or absorbing states.  Since the
class of value functions is not altered, this method has two main
advantages.  First, there are no additional model assumptions compared
to the standard---that is non-pessimistic---version of the
actor-critic method.  Second, the complexity of the underlying classes
is not increased, thereby allowing us to construct tight confidence
intervals and estimation error bounds that are minimax optimal up to
logarithmic factors.

The remainder of this paper is organized as follows.  We begin in
Section~\ref{SecBackground} with background on MDPS, and then
introduce the modeling assumptions that underlie the analysis of this
paper.  In Section~\ref{SecPacle}, we introduce the algorithm studied
in this paper, namely the Pessimistic Actor Critic for Learning without
Exploration
(for short, \Algname{}) algorithm.  Section~\ref{SecMain} provides
statements of our main results and discussion of their consequences,
including an upper bound on the \Algname{} algorithm in
Theorem~\ref{thm:MainResult}, and a minimax lower bound in
Theorem~\ref{thm:LowerBound}.  In Section~\ref{SecProof}, we provide
an outline of the proof of Theorem~\ref{thm:MainResult}, with various
technical details as well as the proof of Theorem~\ref{thm:LowerBound}
deferred to the appendices. We conclude with a discussion in
Section~\ref{SecDiscussion}.


\subsection{Notation}
\label{sec:Notation}

For the reader's convenience, we summarize here some notation used
throughout the paper.  We let $\B_d(r) = \{ x \in \R^d \mid \|x\|_2
\leq r\}$ denote the Euclidean ball of radius $r \in \R$ in dimension
$d$; we simply write $\B$ when there is no possibility of confusion.
For a vector $x \in \R^d$, we use $[x]_i$ to denote its $i^{th}$
component.  We use the $\widetilde O$ notation to denote an upper
bound that holds up to constants and log factors in the input
parameters $(\frac{1}{\delta},d,H)$.  The notation $\lesssim$ means an
upper bound that holds up to a constant, with an analogous definition
for $\gtrsim$.


\section{Background and problem formulation}
\label{SecBackground}

We begin by providing some background, before introducing
the assumptions that underlie our problem formulation.


\subsection{Markov decision processes}

In this paper, we focus on finite-horizon Markov decision processes,
for which we provide a very brief introduction here.  See the
books~\cite{puterman1994markov,bertsekas1996neuro,Bertsekas_dyn1,Bertsekas_dyn1,SutBar18}
for more background and detail. A finite-horizon MDP is specified by a
positive integer $\horizon$, and events take place over a sequence of
stages indexed by the time step $\hstep \in [\horizon] \defeq \{1,
\ldots, \horizon \}$.  The underlying dynamics involve a state space
$\StateSpace$, and are controlled by actions that take values in some
action set $\ActionSpace$.  In this paper, we allow the state space to
be arbitrary (continous or discrete), whereas our analysis applies to
discrete action spaces.  For each time step $\hstep \in [\horizon]$,
there is a reward function $r_\hstep: \StateSpace \times \ActionSpace
\rightarrow \R$, and for every time step $\hstep$ and state-action
pair $(\state, \action)$, there is a transition function
$\trans_\hstep(\cdot \mid \state, \action)$.  When at horizon
$\hstep$, if the agent takes action $\action$ in state $\state$, it
receives a random reward drawn from a distribution $R_\hstep(\state,
\action)$ with mean $r_\hstep(\state, \action)$, and it then
transitions randomly to a next state $\nextstate$ drawn from the
transition function $\trans_\hstep(\cdot \mid \state, \action)$.

A policy $\pi_\hstep$ at stage $\hstep$ is a mapping from the state
space $\StateSpace$ to the action space $\ActionSpace$.  Given a full
policy $\pi = (\pi_1, \ldots, \pi_\horizon)$, the state-action value
function at time step $\hstep$ is given by
\begin{align}
Q^{\pi}_\hstep(\state, \action) & = r_\hstep(\state, \action) +
\E_{\MyState_\ell \sim \pi \mid (\state, \action)} \sum_{\ell = \hstep
  + 1}^{\horizon} r_{\ell}(\MyState_{\ell},
\pi_{\ell}(\MyState_\ell)),
\end{align}
where the expectation is over the trajectories induced by $\pi$ upon
starting from the pair $(\state, \action)$. When we omit the starting
state-action pair $(\state, \action)$, the expectation is intended to
start from a fixed state denoted by $\state_1$.  The value function
associated to $\pi$ is $V^{\pi}_\hstep(\state) =
Q^{\pi}_\hstep(\state, \pi_\hstep(\state))$.  For a given policy
$\pi$, we define the Bellman evaluation operator
\begin{align*}
\T^{\pi}_\hstep(Q_{h+1})(\state, \action) = r_\hstep(\state, \action) +
\E_{\MyState' \sim \trans_\hstep(\state, \action)} \E_{\Action' \sim \pi}
Q_{h+1}(\MyState',\Action').
\end{align*}
Under some regularity
conditions~\cite{puterman1994markov,shreve1978alternative}, there
always exists an optimal policy $\pi^\star$ whose value and
action-value functions are defined as
\begin{align*}
  \Vstar_\hstep(\state) = V^{\pi^\star}_\hstep(s) = \sup_{\pi}
  V^{\pi}_\hstep(\state), \quad \mbox{and} \quad \Qstar_\hstep(\state,
  \action ) = Q^{\pi^\star}_\hstep(\state, \action) = \sup_{\pi}
  Q^{\pi}_\hstep(\state, \action).
\end{align*}


\subsection{Assumptions on data generation}

In this paper, we study a model in which we observe a dataset of the
form $\DataSet = \{ (\state_i, \action_i, \reward_i, \nextstate_i)
\}_{i=1}^\Nsamp$, where $\Nsamp$ is the total sample size.  For each
$i \in [\Nsamp] = \{1, 2, \ldots, \Nsamp \}$, the tuple $(\state_i,
\action_i)$ corresponds to a state-action pair associated with some
time step $\hstep_i$.  We let $\Fcal_i$ be the $\sigma$-field
generated by the samples $\{ (\state_j, \action_j, \reward_j,
\nextstate_j) \}_{j=1}^{i-1}$ that are in the ``past'' relative to
index $i$. With this notation, we impose the following condition:
\begin{assumption}[Data generation]
  \label{AssData}
For each $i \in [\Nsamp]$, the pair $(\state_i, \action_i)$ is
measurable with respect to $\Fcal_i$.  Conditionally on a given pair
$(\state_i, \action_i)$, the random variable $\reward_i$ is drawn from
a reward distribution $R_{\hstep_i}(\state_i, \action_i)$ that is
$1$-sub-Gaussian; and the next state $\nextstate_{i}$ is drawn from
the distribution $\trans_{h_i}(\state_{i}, \action_{i})$.
\end{assumption}
Note that the measurability condition allows the choice of $(\state_i,
\action_i)$ to depend arbitrarily on any of the past data with indices
$j < i$.  The mild assumption allows for considerable freedom.  For
example, the state-action pairs may be chosen from (mixture) policies,
or they can be generated by an adversarial procedure that changes the
data acquisition strategy as feedback is received.

For each $\hstep \in [\horizon]$, we let $\Index_\hstep$ denote the
subset of observation indices $i \in [\Nsamp]$ such that $\hstep_i =
\hstep$.  These index sets define the sub-datasets $\DataSet_\hstep =
\{ (\state_i, \action_i, \reward_i, \nextstate_i), i \in \Index_\hstep
\}$ associated with all samples that are based on state-action pairs
at time step $\hstep$.  We define $\Nsamp_\hstep = |\DataSet_\hstep|$,
so that our total sample size can be written as $\Nsamp =
\sum_{\hstep=1}^\horizon \Nsamp_\hstep$.


\subsection{Policy and function classes}

Next we define the policy space $\Pi$ and the action value function
space $\Q$ over which we seek solutions.  Let \mbox{$\phi: \StateSpace
  \times \ActionSpace \mapsto \R^d$} be a $d$-dimensional feature
mapping.  We assume throughout that these feature mappings are
normalized such that $\|\phi(\state, \action)\|_2 \leq 1$ uniformly
for all $(\state, \action)$-pairs.  We consider action-value functions
that are linear in $\phi$, and families of the form
\begin{subequations}
  \begin{align}
\label{EqnDefnQclass}    
  \Q(\rad) & \defeq \{ (\state, \action) \mapsto \inprod{\phi(\state,
    \action)}{w} \mid \|w\|_2 \leq \rad\},
  \end{align}
  where $\rad \in (0,1]$ is a user-defined radius.  For policies, we
    consider the associated soft-max class
    \begin{align}
\label{EqnDefnSoftMax}      
  \PiSoft (\radactor) \defeq \left \{ \frac{ e^{\inprod{\phi(\state,
        \action)}{\theta}}}{\sum \limits_{\action' \in \ActionSpace}
    e^{\inprod{\phi(\state,\action')}{\theta}}} \; \mid \;
  \|\theta\|_2 \leq \radactor \right \},
\end{align}
\end{subequations}
where $\radactor > 0$ is a second radius.

In the context of our actor-critic algorithm, the weight radius $\rad$
remains fixed for all updates.  On the other hand, the actor produces
a sequence of soft-max radii $\{\radactor_\actorit
\}_{\actorit=1}^\numepoch$, indexed by the iterations $\actorit$ of
the actor.  This sequence is produced via the update rule in Line
\ref{line:actorupdate} of \cref{alg:actor}.  The policy radius can be
large $\radactor \gg 1$ but we constrain $\rad \leq 1$ so that the
critic's estimate $Q_w(\state, \action) = \inprod{\phi(\state,
  \action)}{w}$ is bounded by one, i.e., $\sup_{(\state, \action, w)}
|Q_w(\state, \action)| \leq 1$.

Recall that our MDP consists of sequence of $\horizon$ distinct
stages.  Our algorithm and theory allows for the possibility of
different feature extractors at each step $\hstep \in [\horizon]$,
even with possibly different dimensions.  Consequently, in
implementing and analyzing the algorithm, there are actually
$\horizon$ (possibly different) functional spaces $\{ \Q_\hstep
\}_{\hstep=1}^\horizon$, along with the associated soft-max policy
classes $ \{ \Pi_\hstep \}_{\hstep=1}^\horizon$.  So as to simplify
notation, we drop the dependence on the radii when referring to the
functional spaces, and implicitly assume that the terminal value
function is zero.


\subsection{A range of function class assumptions}

In this section, we discuss a range of assumptions that might be
imposed on the class of action-value functions.  This discussion
serves as motivation for the particular assumption (Bellman restricted
closedness---cf.  Assumption~\ref{asm:Closedness}) that underlies our
analysis.

We begin with the least restrictive condition, which is a very natural
starting point in our given set-up.  If we seek to find the policy
$\pi \in \Pi$ with the highest value function, it seems reasonable to
require that the following representation condition (approximately)
holds.

\begin{assumption}[Linear action-value functions $Q^\pi$]
The MDP admits a linear action-value function representation for all
policies in $\Pi$, meaning that for each policy $\pi \in \PolicySet$
and time step $\hstep \in [\horizon]$, there exists a vector
$w_\hstep^\pi$ such that
\begin{align}
Q_\hstep^\pi(\state, \action) = \inprod{\phi_\hstep(\state,
  \action)}{w_\hstep^\pi}.
\end{align}
\end{assumption}
\noindent This assumption alone turns out to be inadequate to ensure
that effective learning is possible; indeed, the recent
papers~\cite{zanette2020exponential,weisz2020exponential} establish
that even under this condition, there are instances that require
exponentially many samples to do better than a random policy. \\

Given this fact, if one is interested in procedures with polynomial
complexity (in both sample size and running time), stronger conditions
need to be imposed.  In general, the Bellman evaluation operator, even
when applied to a linear action-value function, will return a
nonlinear value function.  The analysis of this paper is based on
bounding the Bellman error in the sense of sup-norm deviation from
linearity:

\begin{assumption}[Bellman Restricted Closedness]
\label{asm:Closedness}
The policy and value function spaces $(\Pi,\Q)$ are closed up to
$\misepsilon \in \R^\horizon$ error in the sup-norm if there is a
non-negative sequence $\{\misepsilon_\hstep \}_{\hstep=1}^\horizon$
such that for each $\hstep \in [\horizon]$, we have
\begin{align}
\sup_{\substack{Q_{h+1} \in \Q_{h+1} \\ \pi_{h+1}\in\Pi_{h+1}}}
\inf_{Q_\hstep \in \Q_\hstep} \| Q_\hstep -
\T_\hstep^{\pi_{h+1}}Q_{h+1}\|_\infty \leq \misepsilon_\hstep.
\end{align}
\end{assumption}

The restricted closedness assumption measures how well we can fit the
action-value function resulting from the application of the Bellman
evaluation operator to an action value function in $\Q$ and for a
policy in $\Pi$. It enables the analysis of least-squares policy
evaluation (e.g.,~\cite{nedic2003least}), which will be our starting
point when constructing the critic. \\
 
Finally, for understanding connections to past work, it is relevant to
compare to the \emph{low-rank MDP} assumption that has been analyzed
in recent work~\cite{jin2020provably,yang2020reinforcement}, including
in offline RL with pessimismistic guarantees~\cite{jin2020pessimism},
as well as in various online
settings~\cite{agarwal2020flambe,modi2021model,zanette2020frequentist}.
\begin{assumption}[Low-Rank MDP]
\label{asm:LowRank}
An MDP is low-rank if for all $\hstep \in [\horizon]$, there exists a
reward parameter $w_\hstep \in \R^{d}$ and a component-wise positive
mapping $\psi_\hstep: \StateSpace \rightarrow \R_+^d$ such that
$\|\psi_\hstep(\state)\|_{1} = 1$ for all $\state \in \StateSpace$,
and
\begin{align*}
r_\hstep(\state, \action) & = \inprod{\phi_\hstep(\state,
  \action)}{w_\hstep} , \qquad \trans_\hstep(s'\mid s,a) =
\inprod{\phi_\hstep(\state, \action)}{\psi_\hstep(s')}, \qquad \forall
(s,a,h,s').  \numberthis{\label{eqn:LinearMDPequations}}
\end{align*}
\end{assumption}

\vspace*{0.05in}

\noindent The following proposition explicates the nested relationship
between these three conditions, showing that the low-rank MDP
condition is the most restrictive: \\
\begin{proposition}[Low Rank $\subset$ Restricted Closedness
    $\subset$ Linear $Q^\pi$]
\label{prop:Inclusion}
For any fixed state-action space, horizon, and feature extractor:
\begin{enumerate}
\item[(a)] The class of low-rank MDPs is a strict subset of the class
  of MDPs that satisfy Bellman restricted closedness.
\item[(b)] The class of MDPs that satisfy Bellman restricted
  closedness is a strict subset of the linear $Q^\pi$ MDP class.
\end{enumerate}
\end{proposition}
\noindent See \Cref{sec:proofofinclusion} for the proof of this
claim.\\

Based on \Cref{prop:Inclusion}, we see that any analysis based on
assuming Bellman restricted closedness also \emph{a fortiori} applies
to MDPs that satisfy the more stringent low-rank MDP condition.


\section{The Pessimistic Actor-Critic}
\label{SecPacle}

Given the set-up thus far, we are now ready to describe the
actor-critic algorithm that we analyze in this paper.  We refer to it
as the \emph{Pessimistic Actor Critic for Learning without
Exploration}, or \Algname{} for short.  We first describe the critic
in \Cref{SecCriticDesc}, and then the actor in \Cref{SecActorDesc}.  We
summarize the actor and critic algorithms, respectively, in pseudocode
form in \Cref{alg:actor} and \Cref{alg:critic}.


\subsection{The Critic: Pessimistic Least Square Policy Evaluation}
\label{SecCriticDesc}

The purpose of the critic is to provide pessimistic value function
estimates corresponding to the policy $\pi$ under consideration by the
actor. Monte Carlo with importance sampling (IS) is not desirable in this setting, as the policy or distribution that generated the dataset might be
unknown and estimation errors on the distribution can accumulate exponentially with the horizon in IS estimators (see e.g.~\cite{liu2018representation}).  Instead, we use a least-squares temporal difference method
for policy evaluation, but suitably perturbed to return pessimistic
estimates---i.e., lower bounds on the true value function of the given
policy $\pi$.  Our method is based on directly perturbing the
regression parameters in the least-square estimate.  In contrast to
bonus-based approaches, this method has the important advantage of
ensuring that the action-value function remains linear.  The purpose
of the perturbations is to compensate for possible statistical errors
in estimating the regression parameter due to poor coverage of the
given dataset.

Let us now give a precise description of the critic.  Given a policy
$\pi = (\pi_1, \ldots, \pi_\horizon)$, the goal of the critic is to
minimize the quantity
\begin{align}
  \label{EqnCriticValue}
 \Ea{\pi_1} \inprod{\phi(\state_1,\Action')}{w_1} & = \sum_{\action
   \in \ActionSpace} \pi_1(\action \mid \state_1)
 \inprod{\phi_1(\state_1,\action)}{w_1},
\end{align}
which is an estimate of the value function $V^\pi(\state_1)$ for the
policy $\pi$ at the initial state $\state_1$.  The parameter $w_1 \in
\R^d$ is a vector to be adjusted, one that is determined by a
backwards-running sequence of regression problems from $\hstep =
\horizon$ down to $\hstep = 1$.

We introduce the pessimistic perturbations directly to the solution of
these regression problems.  They involve a norm defined by the
cumulative covariance matrix.  Recall that $\Index_\hstep$ indexes the
subset of observations associated with state-action pairs at time step
$\hstep$.  For each $\hstep \in [\horizon]$ and $i \in \Index_\hstep$,
let us write the associated sample as the quadruple $(\state_{hi},
\action_{hi}, \reward_{hi}, \state_{h+1, i})$.  Introducing the
shorthand notation $\phi_{hi} = \phi_\hstep(\state_{hi},a_{hi})$, we
define the \emph{cumulative covariance matrix}
\begin{align}
\label{EqnCovariance}
\Sigma_\hstep & \defeq \Big( \sum_{i \in \Index_\hstep} \phi_{hi}
\phi_{hi}^\top \Big) + I_{d \times d},
\end{align}
where $I_{d \times d}$ denotes the $d$-dimensional identity matrix.
Notice that the cumulative covariance grows as the number of samples
in $\Index_\hstep$ increases; we do not normalize it by the local
sample size $\Nsamp_\hstep = |\Index_\hstep|$, so that $\Sigma_\hstep$
effectively represents the amount of information contained in the
sub-dataset $\DataSet_\hstep$ at time step $\hstep$.

Since $\Sigma_\hstep$ is strictly positive definite by construction, it
defines a pair of norms
\begin{align}
  \Signorm{u} \defeq \sqrt{u^\top \Sigma_\hstep u}, \quad \mbox{and} \quad
  \Signorminv{u} \defeq \sqrt{u^\top (\Sigma_\hstep)^{-1} u}.
\end{align}
Consider the regression problem that is solved in moving backward from
time step $\hstep + 1$ to $\hstep$.  Given the weight vector
$w_{\hstep + 1}$ at time step $\hstep + 1$, the regularized
least-squares estimate of $w_\hstep$ is given by
\begin{align*}
  \widehat{w}_\hstep & \defeq \Sigma^{-1}_{h} \sum_{k \in
    \Index_\hstep} \phi_{h k} \Big[r_{hk} + \sum_{\action
      \in \ActionSpace} \pi_{\hstep + 1}(\action \mid \state_{h+1,k})
    \inprod{\phi_{h+1}(\state_{h+1,k}, \action)}{w_{h+1}} \Big].
\end{align*}
We introduce pessimism by directly perturbing the weight vectors
themselves---that is, we search for weight vector $w_\hstep$ such that
$w_\hstep = \xi_\hstep + \widehat{w}_\hstep$, where the pessimism vector
$\xi_{h} \in \R^d$ satisfies a bound of the form $\Signorm{\xi_\hstep} \leq
\pess_\hstep$, for a user-defined parameter $\pess_\hstep$.

In detail, the critic takes as input the dataset $\DataSet$, a policy
$\pi$, a sequence of tolerance parameters \mbox{$\pess = (\pess_1,
  \ldots, \pess_\horizon)$,} weight radii \mbox{$\rad = (\rad_1,
  \ldots, \rad_\horizon)$} with each $\rad_\hstep \in (0,1]$.  The
  optimization variables consist of the regression vectors \mbox{$w =
    (w_1, \ldots, w_\horizon) \in (\R^d)^\horizon$} and the pessimism
  vectors \mbox{$\xi = (\xi_1, \ldots, \xi_\horizon) \in
    (\R^d)^\horizon$.}  The critic then solves the convex program
  \begin{subequations}
\label{EqnCriticProgram}    
  \begin{align}    
(\xi^\pi, \wubar^\pi) \defeq \arg \min_{ \substack{ \xi \in
      (\R^d)^\horizon \\ w \in (\R^d)^\horizon}} \sum_{\action
    \in \ActionSpace} \pi_1(\action \mid \state_1)
  \inprod{\phi_1(\state_1, \action)}{w_1}
\end{align}
with the terminal condition $w_{\horizon + 1} = 0$, and subject to the
constraints
\begin{align}
  \label{EqnCriticConstraint}
  w_\hstep & = \xi_\hstep + \Sigma^{-1}_{h} \sum_{k \in \Index_\hstep} \phi_{h
    k}\Big[ r_{hk} + \sum_{\action \in \ActionSpace} \pi_{h+1}(a \mid
    \state_{h+1,k}) \inprod{\phi_{h+1}(\state_{h+1,k},a)}{w_{h+1}}
    \Big], \qquad \mbox{and} \\
& \|\xi_\hstep \|^2_{\Sigma_{h}} \leq \pess^2_\hstep, \qquad \|w_\hstep\|^2_{2} \leq
  (\rad_\hstep)^2
\end{align}
\end{subequations}
for all $\hstep \in [\horizon]$.  Here the matrices $\Sigma_\hstep$
were previously defined in equation~\eqref{EqnCovariance}.

The convex program~\eqref{EqnCriticProgram} consists of a linear
objective subject to quadratic constraints; it is a special case of a
second order cone program, and can be efficiently solved with standard
convex solvers.

\begin{center}
\begin{minipage}{0.45\textwidth}
\begin{algorithm}[H]
\caption{\textsc{Actor (Mirror Descent)}}
\label{alg:actor}
\begin{algorithmic}[1]
  \State \textbf{Input}: Dataset $\D$, starting state $\state_1$, learning rate $\eta$
  \State Set $\theta_{1} = (\vec 0,\dots,\vec 0)$ \For{$\actorit = 1,
    2,\ldots, \actortot$} \State $\wubar_\actorit \leftarrow$
  \textsc{Critic}$(\D,\pi_{\theta_\actorit},\state_1)$
  \State \label[line]{line:actorupdate} $\theta_{\actorit + 1} =
  \theta_\actorit + \eta \wubar_\actorit$ \EndFor \State
  \textbf{Return: Mixture policy $\pi_{\theta_1}, \ldots,
    \pi_{\theta_\actortot}$}
\end{algorithmic}
\end{algorithm}
\end{minipage}
\hfill
\begin{minipage}{0.45\textwidth}
\begin{algorithm}[H]
\caption{\textsc{Critic (Plspe)}}
\label{alg:critic}
\begin{algorithmic}[1]
\State \textbf{Input}: Dataset $\DataSet$, target policy $\pi$,
starting state $\state_1$, critic radii $\{\rad_\hstep\}_{\hstep = 1,\dots,\horizon}$, and parameters $\{ \pess_\hstep\}_{\hstep = 1,\dots,\horizon}$
\State Solve the optimization
program~\eqref{EqnCriticProgram}
\State \textbf{Return:} Optimal weight vector $\wubar$
\end{algorithmic}
\end{algorithm}
\end{minipage}
\end{center}


\subsection{The Actor: Mirror Descent}
\label{SecActorDesc}

We now turn to the behavior of the actor.  It applies the mirror
descent algorithm based on the Kullback Leibler (KL)
divergence~\cite{bubeck2014convex}. This combination leads to the
exponentiated gradient update rule in every timestep $\hstep \in
[\horizon]$, so that the soft-max policy in moving from iteration
$\actorit$ to $\actorit + 1$ is updated as
\begin{align}
\label{main:eqn:MirrorDescent}
\pi_{\actorit + 1,\hstep}(\action \mid \state) \propto
\pi_{\actorit,\hstep}(\action \mid \state) e^{\eta Q_\hstep(\state,
  \action)} \qquad \mbox{for each $(\state, \action) \in \StateSpace
  \times \ActionSpace$.}
\end{align}
Here $\eta > 0$ is a stepsize parameter, and our theory specifies
a suitable choice.

If the $Q$-value above from the critic lives in $\Q$, then it is
possible to show that $\pi_{\actorit + 1,\hstep} \in \Pi_\hstep$ and
the update rule takes a much simpler and computationally more
efficient form (cf. Line \ref{line:actorupdate} of \cref{alg:actor}),
where $\wubar_{\actorit}$ is the gradient of the value function on the
pessimistic MDP implicitly identified by the critic.  In this case,
the spaces $(\Q,\Pi)$ are said to be \emph{compatible}
\cite{sutton1999policy,kakade2001natural,agarwal2020optimality,raskutti2015information}
and the resulting algorithm is often called the \emph{Natural Policy
Gradient} (NPG) (see also \cite{geist2019theory,shani2020adaptive}).
By construction, the critic maintains a linear action value function
even after pessimistic perturbations.  As a consequence, the actor
policy space is the simple softmax policy class $\Pi$ and the easier
update rule can be used. As we explain in the analysis, this has
important statistical benefits.

After $\actortot$ rounds of updates, the mirror descent algorithm that
we use here readily achieves online regret rates (in the optimization
setting with exact feedback) $\sim 1/\actortot$ or $\sim
1/\sqrt{\actortot}$ depending on the
analysis~\cite{agarwal2020optimality} and the learning rate, although
we mention that these rates could potentially be
improved~\cite{khodadadian2021linear,lan2021policy,bhandari2020note}.


\section{Main results}
\label{SecMain}

We now turn to the statement of a bound on the performance of the
policy $\pialg$ returned by \Algname{}.  This upper bound involves
three terms: an optimization error, an uncertainty term, and a model
mis-specification term.  The \emph{optimization error} is given by
\mbox{$\Ropt(\actortot) \defeq 4 \horizon \sqrt{\frac{\log
      |\ActionSpace|}{\actortot}}$;} it captures the rate at which the
error decreases as a function of the iterations of the actor.  The
\emph{mis-specification error} \mbox{$\Mis(\misepsilon) \defeq \sumh
  \misepsilon_\hstep$} is simply the sum of all the stage-wise
mis-specification errors; notice that the mis-specification error does
depend on the choice of the radii for the critic
$\rad_1,\dots,\rad_\horizon$ in a problem dependent way
(cf. \cref{asm:Closedness}).  Finally, for each $\hstep$, define the
vector $\meanpivec_\hstep \defeq \E_{(\MyState_\hstep, \Action_\hstep)
  \sim \pi} [\phi_\hstep(\MyState_\hstep, \Action_\hstep)]$, where the
expectation is over the state-action $(\MyState_\hstep,
\Action_\hstep)$ encountered at timestep $\hstep$ upon following
policy $\pi$.  In terms of these vectors, the \emph{uncertainty error}
is given by
\begin{align}
\label{main:eqn:UR}
\Uncertain(\pi; \pess) & \defeq \; 2 \sumh \pess_\hstep
\|\meanpivec_\hstep\|_{\Sigma^{-1}_\hstep} \; = \; 2 \sumh \pess_\hstep \sqrt{
  (\meanpivec_\hstep)^\top \Sigma^{-1}_\hstep \meanpivec_\hstep},
\end{align}
where the cumulative covariance matrix $\Sigma_\hstep$ was defined in
equation~\eqref{EqnCovariance}.

The amount of information from the dataset $\DataSet$ is fully encoded
in the uncertainty function $\Uncertain$ through the sequence of
cumulative covariance matrices $\{\Sigma_\hstep\}_{\hstep=1}^\horizon$
and parameters $\{ \pess_\hstep\}_{\hstep=1}^\horizon$.  The more data
are available, the more positive definite $\Sigma_\hstep$ is and the
smaller the uncertainty function $\Uncertain(\pi; \pess)$ becomes for
a fixed policy $\pi$. If the sampling distribution that generates the
dataset is fixed, then we can write $\Uncertain(\pi; \pess)
\lessapprox c/\sqrt{n}$ where $c$ does not depend on $n$ and can be
interpreted as the coverage of the sampling distribution with respect
to policy $\pi$.

\subsection{A guarantee for PACLE}
Our main result holds under Assumption~\ref{AssData} on the data
collection process. It is based on radii $\{ \rad_\hstep
\}_{\hstep=1}^\horizon$ for the action value function\footnote{This
represents a setting where both the reward and the value function can
be as large as $1$ in absolute value. One easily recovers the setting
with value functions in $[0, \horizon]$ using a rescaling argument.}
that lie in the interval $(0,1]$, and it provides a guarantee relative
  to the class $\Piall$ of all stochastic policies.
\begin{theorem}[An achievable guarantee]
\label{thm:MainResult}
Suppose that we are given a data set $\DataSet$ collected in a way
that respects Assumption~\ref{AssData}.  Then there are pessimism
vectors bounded as $\pess_\hstep = \widetilde O(\sqrt{d \; \log(1/\delta)})
+ \misepsilon_\hstep \sqrt{\Nsamp_\hstep}$ such that, after running $\actortot
\geq \log |\ActionSpace|$ rounds of the actor with stepsize $\eta =
\sqrt{ \tfrac{\log |\ActionSpace|}{\actortot}}$, the \Algname{}
procedure returns a policy $\pi_{\Alg{}}$ for which
\begin{align}
\label{eqn:MainResult}
V^\pi_1(\state_1) - V_1^{\pi_{\Alg}}(\state_1) & \leq \Uncertain(\pi;
\pess) + \underbrace{\sum_{\hstep=1}^\horizon
  \misepsilon_\hstep}_{\Mis(\misepsilon)} +
\underbrace{\vphantom{\sum_{\hstep=1}^\horizon}4 \horizon
  \sqrt{\tfrac{\log |\ActionSpace|}{\actortot}}}_{\Ropt(\actortot)}
\qquad \mbox{uniformly over all $\pi \in \Piall$}
\end{align}
with probability exceeding $1 - \delta$.
\end{theorem}

The result provides a family of upper bounds on the sub-optimality of
the learned policy $\pialg$, indexed by the choice of comparator
policy $\pi$, and embodies a tradeoff between the sub-optimality of
the comparator $\pi$ and its uncertainty $\Uncertain(\pi; \pess)$.
Note that the optimization error $\Ropt(\actortot)$ can be reduced
arbitrarily, while $\pess$ (and thus $\Uncertain(\pi; \pess)$)
increase only logarithmically with $\actortot$.
As a special case, if we set $\pi = \pistar$ and assume that there is
no mis-specification error, then we obtain that the learned policy
satisfies a bound of the form
\begin{align}
\label{EqnOptGuarantee}
V_1^{\pistar}(\state_1) - V_1^{\pialg}(\state_1) & \leq
\Uncertain(\pistar; \pess) + \Ropt(\actortot)
\end{align}
with probability at least $1-\delta$.  Since $\Ropt(\actortot)$ is
well-controlled, this guarantee is satisfied  whenever the uncertainty
term $\Uncertain(\pistar; \pess)$ is small.

More generally, the guarantee~\eqref{eqn:MainResult} is significantly
stronger than most prior work as \Algname{} competes not just with the optimal policy
$\pistar$, but with all comparator policies simultaneously.  Such
comparator policies need not necessarily be in the prescribed policy
class $\Pi$.  To highlight the strength of this generality, suppose
that the uncertainty $\Uncertain(\pistar; \pess)$ of the optimal
$\pistar$ is \emph{not} small---it could in fact be infinite. In this
case, the bound~\eqref{EqnOptGuarantee} would not be useful.

However, suppose that there exists a near-optimal policy---meaning a
policy $\pi^+$ such that $V_1^{\pi^+}(\state_1) \geq
\Vstar_1(\state_1) - \epsilon$ for some small $\epsilon$---that is
well-covered by the dataset (i.e., for which $\Uncertain(\pi^+; \pess)
\approx 0$).  In this case, \cref{thm:MainResult} ensures with high
probability $V_1^{\Alg{}}(\state_1) \gtrsim \Vstar_1(\state_1) -
\epsilon$. In contrast, traditional analyses that use only the optimal
policy $\pistar$ as a comparator---as opposed to also allowing
near-optimal policies---cannot return meaningful guarantees.  We note
also that the papers~\cite{yu2020mopo,liu2020provably,kidambi2020morel} provide results
of a similar flavor. 
These types of guarantees are also provided by some concurrent works \cite{uehara2021pessimistic,xie2021bellmanconsistent}.

It should also be noted that \Cref{thm:MainResult} provides a family
of results indexed by the choice of the critic's radii $\{\rad_\hstep
\}_{\hstep=1}^\horizon$.  This choice is a modeling decision:
increasing the radii increases both the approximation power of the
function class $\Q_\hstep$ used for regression, but also increases the
complexity of the function class $\Q_{h+1}$ to represent (cf.
\cref{asm:Closedness}); thus, the choice of the radii affects the
approximation error $\Mis(\misepsilon)$ in a problem dependent way.

\subsection{A lower bound}

Thus far, we have stated an upper bound on the quality of the returned
policy for a given procedure.  Central to this upper bound is the
uncertainty function $\Uncertain(\pi; \pess)$.  In this section, we
show that a term of this form is unavoidable for any procedure.  In
particular, working within the well-specified setting, we prove a
lower bound in terms of the quantity $\Uncertain(\pi; \sqrt{d}) =
\sqrt{d} \sumh \|\meanpivec_\hstep\|_{\Sigma^{-1}_\hstep}$.  Recalling that
our choice of $\pess$ scales with $\sqrt{d}$ (along with other
logarithmic factors), this lower bound shows that our result is tight
up to logarithmic factors.

We show that the lower bound actually holds in a setting that is
easier for the learner, in the sense that (1) we restrict to low-rank
MDPs, where there is no mis-specification error; and (2) the mechanism
that generates the dataset is non-adaptive, and so certainly satisfies
Assumption~\ref{AssData}.

\begin{theorem}[Information-theoretic lower bound]
  \label{thm:LowerBound}
  For a given horizon $\horizon$ and dimension $d$, consider a sample
  size $\Nsamp \geq 2d^3\horizon^3$. There is a class $\MDPClass$ of
  low-rank MDPs and a data generating procedure satisfying
  Assumption~\ref{AssData} such that for any policy $\piest$, we have
\begin{align}
\sup_{M \in \MDPClass} \E_M \left[ V^\pi_{1M}(\state_1) -
  V^{\piest}_{1M}(\state_1) \right] & \geq c \; \Uncertain(\pi; \sqrt{d})
\qquad \mbox{uniformly over all $\pi \in \Piall$,}
\end{align}
where $c > 0$ is a universal constant.
\end{theorem}
When $\horizon = 1$ the above result gives a sample complexity lower bound for learning a near optimal policy from batch data in a linear bandit instance.

\subsection{Comparison to related work}

\cref{thm:MainResult} automatically implies the typical bound 
\mbox{$\Pro[V_1^{\pi_{\Alg}}(\state_1) \geq \Vstar_1(\state_1) -
    \Uncertain(\pistar; \pess) ] \geq 1-\delta$} when the comparator
policy is the optimal policy $\pistar$, e.g.,
\cite{jin2020pessimism,rashidinejad2021bridging,kidambi2020morel,kumar2019stabilizing,buckman2020importance}.
The guarantee can be written as $V_1^{\pialg}(\state_1) \gtrsim
\Vstar_1(\state_1) - C/\sqrt{n}$ where $n$ is the number of samples
and $C$ is the (scaled) condition number of $\Sigma^{-1}_{h}$. One
could interpret $C$ as a concentrability coefficient that expresses
the coverage of dataset---through $\Sigma_\hstep$---with respect to the
average direction in feature space $\E_{(\MyState_\hstep, \Action_\hstep) \sim
  \pistar_\hstep} [\phi(\MyState_\hstep, \Action_\hstep)]$ of the optimal policy
$\pistar$. As in the paper~\cite{jin2020pessimism}, such a factor can
be small even when traditional concentrability coefficients are large
because they depend on state-action visit ratios (see the literature
in \cref{sec:Literature}, e.g., \cite{chen2019information}).

With reference to the results in the paper~\cite{jin2020pessimism},
our work provides improvements in two distinct ways.  First, their
upper and lower bounds exhibit a gap of the order $d \horizon$, which
our analysis closes.  Second, our analysis holds under the more
permissive \fullref{asm:Closedness} which includes low-rank MDPs.  Of
this improvement, a factor of $\sqrt{d}$ is due to the algorithm that
we use, and the remainder is due to a more refined  construction to certify
optimality in \cref{thm:LowerBound}.  To be clear, our upper and lower
bounds differ from theirs by a factor of $H$ due to a different
normalization in the value function).  We also note that the result of
Liu et al.\cite{liu2020provably} can be specialized to the low-rank
MDP setting; however, even in this simpler setting, the results would
be sub-optimal and also require additional density estimates.
 
Deriving a computationally tractable model-free algorithm without
low-rank dynamics but subject to value function perturbations (e.g.,
optimistic or pessimistic perturbations) is an open problem even in
the more heavily studied online exploration setting: there the current
state-of-the
art~\cite{zanette2020learning,jin2021bellman,du2021bilinear,jiang17contextual}
only present computationally \emph{intractable} algorithms with the
exception of \cite{zanette2020provably} for a PAC setting with low
inherent Bellman error which however requires an additional
``explorability'' condition.


\section{Proofs}
\label{SecProof}

In this section, we provide an outline of the proof of
Theorem~\ref{thm:MainResult}.  The main components of the proof are
guarantees for the pessimistic estimates produced by the critic, and
online learning guarantees for the updates taken by the actor.  These
two guarantees are coupled together via the notion of an induced MDP.

The proof outline given here follows a bottom-up approach: (a)
starting with the critic in Section~\ref{SecCriticOutline}, we first
introduce the notion of induced MDP that links the critic's output to
the actor's input (see Section~\ref{SecInduced}), and then discuss how
suitable choices of the pessimism parameters $\pess$ allow us to
guarantee that the critic underestimates the true value function (see
Sections~\ref{SecGoodEvent} and~\ref{SecPessChoice}); (b) next in
Section~\ref{SecActorOutline}, we provide online-style learning
guarantees for the actor, again using the notion of induced MDP to
link these guarantees back to the critic; and (c) in
Section~\ref{SecCombineOutline}, we put together the pieces to prove
the theorem itself.


\subsection{Critic's Analysis}
\label{SecCriticOutline}

Given a policy $\pi$ and pessimism parameters $\pess$ for which the
convex program~\eqref{EqnCriticProgram} is feasible, the critic
returns the pair $(\xiunder^\pi, \wubar^\pi) = \{(\xiunder^\pi_\hstep,
\wubar^\pi_\hstep)\}_{\hstep=1}^\horizon$.  These weight vectors
induce the estimated value functions
\begin{align}
  \label{EqnCriticEstimates}
\Qubar_\hstep^\pi(\state, \action) \defeq
\inprod{\phi(\state,\action)}{\wubar^\pi_\hstep}, \qquad \mbox{and} \quad
\Vubar_\hstep^\pi(\state) \defeq \Ea{\pi_{h}(\cdot \mid \state)}
\Qubar_\hstep^\pi(\state, \Action').
\end{align}
Our goal in analyzing the critic is to relate these critic-estimated
value functions to the true value functions $\{Q_\hstep^\pi
\}_{\hstep=1}^\horizon$.

\subsubsection{Induced MDP}
\label{SecInduced}

Essential to our analysis is an object that provides the essential
link between the critic's output and the actor's input.  In
particular, it is helpful to understand the critic in the following
way: when given a policy $\pi$ as input, the critic computes the
estimates $\{\Qubar^\pi_\hstep \}_{\hstep=1}^\horizon$, and uses them
form a new MDP $\IndMDP{\pi}$, which we refer to as the \emph{induced
MDP}.  This new MDP shares the same state/action space and transition
dynamics with the original MDP $\MDP$, differing only in the
perturbation of the reward function.  In particular, for each $\hstep
\in [\horizon]$, we define the \emph{perturbed reward function}
\begin{align}
\label{EqnRewardHat}
\rewardhat^\pi_\hstep(\state, \action) & \defeq \reward_\hstep(\state,
\action) + \Qubar_{\hstep}^\pi(\state, \action) -
\T^\pi_\hstep(\Qubar_{\hstep+1}^\pi)(\state, \action).
\end{align}
The induced MDP $\IndMDP{\pi}$ is simply the original MDP that uses
this perturbed reward function.

One important property of the induced MDP---which motivates the
definition~\eqref{EqnRewardHat}---is that the
estimates~\eqref{EqnCriticEstimates} returned by the critic correspond
to the \emph{exact value functions} of policy $\pi$ in the induced
MDP.  We summarize in the following:
\begin{lemma}[Critic exactness in induced MDP]
\label{LemCriticExact}  
  Given a policy $\pi$ as input, the critic returns a sequence
  $\{\Vubar^\pi_\hstep \}_{\hstep=1}^\horizon$ such that
\begin{subequations}
  \begin{align}
\label{EqnCriticExact}      
\Qubar^\pi_\hstep & = Q^\pi_{\hstep, \IndMDP{\pi}}, \quad \mbox{and}
\\
\Vubar^\pi_\hstep & = V^\pi_{\hstep, \IndMDP{\pi}} \qquad \mbox{for
  all $h \in [\horizon]$,}
  \end{align}
\end{subequations}
where $V^\pi_{\hstep, \IndMDP{\pi}}$ is the exact value function of
policy $\pi$ in the induced MDP $\IndMDP{\pi}$.
\end{lemma}
\noindent See Section~\ref{SecProofLemCriticExact} for the proof of
this claim.\\

Moreover, since the induced MDP differs from the original MDP only in
terms of the reward perturbation~\eqref{EqnRewardHat}, we have the
following convenient property: for any policy $\pitil$---which need
not be of the soft-max form---the definition of value functions
ensures that
\begin{align}
  \label{EqnSimpleValue}
 V^{\pitil}_{1,\IndMDP{\pi}}(s_1) - V^{\pitil}_{1}(s_1) & = \sumh
 \E_{(\MyState_\hstep, \Action_\hstep)\sim \pitil}
 \Big[\rewardhat^\pi_\hstep(\MyState_\hstep, \Action_\hstep) -
   \reward_\hstep(\MyState_\hstep, \Action_\hstep) \Big],
\end{align}
where $V^{\pitil}_{1, \IndMDP{\pi}}$ is the value function of $\pitil$
in the induced MDP.  This simple relation allows us to use the
induced MDP to relate arbitrary policies to their exact value
functions.

\subsubsection{Critic's guarantee under a ``good'' event}
\label{SecGoodEvent}
We now show that there is a ``good event''---call it
$\GoodEvent(\pess)$---under which the critic's value function
estimates have some additional desirable properties.  Once this event
is defined, the core of our proof involves determining the smallest
choice of pessimism parameters under which it holds with probability
at least $1 - \delta$.

We begin with some notation required to define the good event.  Let
$\Fclass$ denote the space of all real-valued functions on
$\StateSpace \times \ActionSpace$.  The \emph{regression operator} is
a mapping from $\Fclass$ to $\R^d$, given by
\begin{subequations}
\begin{align}
\label{EqnDefnRegressionOperator}  
  \regest_\hstep^\pi(F) & \defeq \Sigma^{-1}_\hstep
  \sum_{k=1}^\numepoch \phi_{hk} \big \{ r_{hk} + \Ea{\pi(\cdot \mid
    \state_{hk})} F(\state_{h+1,k}, \Action') \big \},
\end{align}
where $F \in \Fclass$.  To appreciate the relevance of the regression
operator, note that
by definition of the critic, 
we have the equivalence
\begin{align}
  \label{EqnCritic2Regression}
\wubar^\pi_\hstep = \xiunder^\pi_\hstep +
\regest^\pi_\hstep(\Qubar^\pi_{\hstep+1}).
\end{align}
We also define the \emph{sup-norm projection operator} (for the definition of $\B$ please see \cref{sec:Notation})
\begin{align}
\label{EqnSupProject}
\SupProj^\pi_\hstep(F) & \defeq \arg \min_{w_\hstep \in \B(\rad_\hstep)}
\sup_{(\state, \action)} \Big | \inprod{\phi(\state, \action)}{w_\hstep} -
\(\T^\pi_\hstep F\)(\state, \action) \Big|.
\end{align}
\end{subequations}
Note that $\SupProj^\pi_\hstep$ is a mapping from $\Fclass$ to $\R^d$;
it returns the weight vector of the best-fitting linear function to
the Bellman update $\T^\pi_\hstep(F)$.

Our good event is defined in terms of the \emph{parameter error
operators} $\TotErrorPlain^\pi_\hstep: \Fclass \rightarrow \R^d$ given
by
\begin{align}
  \label{EqnDefnParError}
\TotError{\pi}{\hstep}(F) & \defeq \regest^\pi_\hstep(F) -
\SupProj^\pi_\hstep(F).
\end{align}
For a given sequence $\pess = (\pess_1, \ldots, \pess_\horizon)$ of
pessimism parameters, we define the \emph{good event}
\begin{align}
\label{EqnGoodEventMain}
\GoodEvent(\pess) & \defeq \Big \{ \sup_{ \substack{Q_{h+1} \in
    \Qlin_{h+1} \\ \pi_{h+1} \in \Pi_{h+1} }} \Signorm{
  \TotError{\pi_{\hstep+1}}{\hstep}(Q_{h+1})} \leq \pess_\hstep
\quad \mbox{for all $h \in [\horizon]$} \Big \}.
\end{align}

\paragraph{Some intuition:}  Why is this event relevant for guaranteeing good
performance of the critic?  In order to gain intuition, let us
consider the special case in which there is no approximation error, so
that the exact state-action value functions are actually linear.
Letting $w_\hstep^\pi$ denote the parameter associated with the linear
action-value function at step $\hstep$, when the good event holds, our
choice of $\pess$ allows us to set
\begin{align*}
  \xiunder^\pi_\hstep & = - \TotError{\pi_{\hstep+1}}{\hstep}(Q^\pi_{h+1})
  \; = \; w_\hstep^\pi - \regest^\pi_\hstep(Q^\pi_{\hstep+1}) \qquad
  \mbox{for each $\hstep \in [\horizon]$,}
\end{align*}
in the constraints~\eqref{EqnCriticConstraint}.  In this way, at each step $\hstep$ the vector $ \xiunder^\pi_\hstep$ can perfectly compensate the noise error $ \TotError{\pi_{\hstep+1}}{\hstep}(Q^\pi_{h+1})$ ensuring that the action-value function $Q^\pi_\hstep$ (compactly encoded in the parameter $w_\hstep^\pi$) can be perfectly represented. In other words, our choice guarantees that the feasible set for ~\eqref{EqnCriticProgram} contains the `true' solution $w_\hstep^\pi$.  Since the convex program
involves minimizing over value functions, this feasibility underlies
showing the critic returns an underestimate of the true value function
for $\pi$ along with some approximation error in the general setting;
see equation~\eqref{EqnCriticPiMain} below for a precise statement. We highlight that such underestimates is only guaranteed at the initial state $\state_1$ and timestep $\hstep = 1$ as encoded in the objective of the program in equation~\eqref{EqnCriticProgram}.

On the other hand, for other policies $\pitil$, we can use the
relation~\eqref{EqnSimpleValue} to control the difference between the
value function $V^{\pitil}_{1,\IndMDP{\pi}}$ in the induced MDP, and
the exact value function $V^{\pitil}_{1}$; see
equation~\eqref{EqnCriticPitilMain} for a precise statement of our
conclusion.  We summarize all of our findings thus far in the
following:
\begin{proposition}
\label{PropCriticMain}
Conditionally on the event $\GoodEvent(\pess)$, when given as input
any policy $\pi$ in the soft-max class $\PiSoft(\softrad)$, the critic
returns an induced MDP $\IndMDP{\pi}$ such that:
\begin{enumerate}
  \begin{subequations}
  \item[(a)]  For the given policy $\pi$, we have
    \begin{align}
      \label{EqnCriticPiMain}
      V^\pi_{1, \IndMDP{\pi}}(\state_1) & \leq V^\pi_1(\state_1) +
      \sumh \misepsilon_\hstep.
    \end{align}
  \item[(b)] For any policy $\pitil$, not necessarily in the soft-max
    class $\Pi$, we have
    \begin{align}
      \label{EqnCriticPitilMain}
\Big| V^{\pitil}_{1,\IndMDP{\pi}}(s_1) - V^{\pitil}_{1}(s_1)\Big| &
\leq 2 \sumh \pess_\hstep \; \Signorminv{ \MeanVec{\pitil}_\hstep } +
\sumh \misepsilon_\hstep,
    \end{align}
  \end{subequations}
  where $\MeanVec{\pitil}_\hstep \defeq \E_{(\MyState_\hstep,
    \Action_\hstep)\sim \pitil} [\phi_\hstep(\MyState_\hstep,
    \Action_\hstep)]$.
  \end{enumerate}
\end{proposition}
\noindent See Appendix~\ref{AppProofPropCriticMain} for the proof.


\subsubsection{Choice of pessimism parameters}
\label{SecPessChoice}

Based on Proposition~\ref{PropCriticMain}, our problem is now reduced
to determining a choice of $\pess$ for which the good
event~\eqref{EqnGoodEventMain} holds with probability at least $1 -
\delta$.  The bulk of our effort in analyzing the critic is devoted to
the technical details of this step; we provide only a high-level
summary here.

The event~\eqref{EqnGoodEventMain} needs to hold uniformly over the
value function and policy classes used by the algorithm.  Our analysis
involves deriving an upper bound $\softrad$ on the $\ell_2$-radius of
the actor parameter over all $\numepoch$ iterations of the algorithm, as follows: 
$$\radactor = \| \theta_\actortot \|_2 = \| \sum_{\actorit=1}^\actortot \eta w_\actorit \|_2 \leq \sum_{\actorit=1}^\actortot \eta\|  w_\actorit \|_2 \leq \actortot \eta \defeq  \softrad. $$
For such choice of $\softrad$ and failure probability
$\delta \in (0,1)$, suppose that we set
\begin{multline}
\label{EqnPessChoice}
\pess_\hstep(\delta) \defeq 1 + \sqrt{\Nsamp_\hstep} \misepsilon_\hstep +
c \left \{1 + d \log \big( 1 + \tfrac{\numepoch}{d} \big) + d\log
\big(1 + 8 \sqrt{\numepoch} \big) + d \log \(1 + 16 \softrad\sqrt{T}
\) + \log \tfrac{\horizon}{\delta} \right \}^{1/2}
\end{multline}
for a suitably large
universal constant $c$. Central to our analysis is the following
lemma:

\begin{lemma}
  \label{LemGoodProbability}
For any $\delta \in (0,1)$, given the choice of pessimism vector
$\pess(\delta)$ in equation~\eqref{EqnPessChoice}, we have
\begin{align}
  \label{EqnGoodProbability}
  \Pro \big[ \GoodEvent(\pess(\delta)) \big] & \geq 1 - \delta.
\end{align}
\end{lemma}
\noindent See Section~\ref{SecProofLemGoodProbability} for the proof
of this claim.

In our proof of Lemma~\ref{LemGoodProbability}, we benefit from the
fact that our procedure injects its pessimism by direct perturbations
of the parameter vectors.  Indeed, one key step in the proof is
bounding certain metric entropies defined by classes $\Q$ of linear
action-value functions, and policy classes $\PiSoft(\softrad)$ used in
the actor's iterations.

First, for any fixed policy $\pi$, since the agent's action value
function $\Qubar^\pi$ is enforced to be linear $\Qubar^\pi \in \Q$
even after perturbations, the relevant action-value class $\Q$ is also
linear.  Thus, we need only control metric entropy (and perform union
bounds over the resulting covering) for a linear function class; in
this way, we avoid a potentially more costly union bound over the much
larger function class obtained by adding complex bonuses to linear
functions, as in past work~\cite{jin2020pessimism}.  In this way, we
achieve a guarantee that is sharper by a factor of $\sqrt{d}$.

Second, the union bound needs to be extended to all policies that the
actor can use to invoke the critic.  Recall that the critic returns a
linear action-value function $\Qubar$, which is
compatible~\cite{kakade2001natural,agarwal2020optimality} with the
soft-max policy class $\PiSoft$.  Consequently, the actor's updates
take the simple form~\eqref{line:actorupdate} of \cref{alg:actor}.  If
the action-value function $\Qubar$ were perturbed by bonuses, then
linearity of the critic's value function would be lost.


\subsection{Actor's Analysis}
\label{SecActorOutline}

In this section, we analyze the mirror descent algorithm---that is,
the actor in \cref{alg:actor}. Our analysis exploits the methods in
the paper~\cite{agarwal2020optimality}, with some small changes to
accommodate our framework; in particular, while our analysis assumes
no error in the critic's evaluation, it does involve a sequence of
time-varying MDPs.

Given a sequence of MDPs $\{M_\actorit\}_{\actorit = 1}^\actortot$,
let $V^\pi_\actorit$ be the value function associated with policy
$\pi$ on MDP $M_\actorit$.  Given the initialization $\theta_1 = 0$,
let $\{ \theta_\actorit\}_{\actorit = 1}^\actortot$ be parameter
sequence generated by the actor, and let $\pi_\actorit =
\pi_{\theta_\actorit}$ be the policy associated with parameter
$\theta_\actorit$.  For each $\actorit$, there is a sequence
\mbox{$w_\actorit = \{w_{h \actorit} \}_{h=1}^\horizon$} such that $\|
w_{h \actorit} \|_2 \leq \rad_\hstep $ for all $h \in
[\horizon]$, and 
\begin{subequations}
\begin{align}
Q^{\pi_\actorit}_{h,M_\actorit}(\state, \action) & \defeq
\inprod{\phi_\hstep(\state, \action)}{w_{h \actorit}}, \qquad
\mbox{for all $(\state, \action)$ and $\hstep \in [\horizon]$.}
\end{align}
In particular, the value of $w_{h \actorit}$ is the value
$\wubar_{\hstep\actorit}$ identified by the critic (see
equation~\eqref{EqnCriticEstimates}) corresponding to policy
$\pi_\actorit$, so that $Q^{\pi_\actorit}_{M_\actorit} =
\Qubar^{\pi_\actorit}$.  Define the value function
\mbox{$V^{\pi_\actorit}_{h,M_\actorit}(\state) = \Ea{\pi_\actorit}
  \big[ Q^{\pi_\actorit}_{h,M_\actorit}(\state, \Action') \big]$}
along with the advantage function
\begin{align}
\label{EqnDefnAdvantage}
 \Advan^{\pi_\actorit}_{h,M_\actorit}(\state, \action) \defeq
 Q^{\pi_\actorit}_{h,M_\actorit}(\state, \action) -
 V^{\pi_\actorit}_{h,M_\actorit}(\state).
\end{align}
\end{subequations}

\begin{proposition}[Actor's Analysis]
  \label{PropActorMain}
Suppose that the actor takes $\actortot \geq \log |\ActionSpace|$
steps using a stepsize $\eta \in (0, 1)$, and the advantage function
at each iteration $\actorit$ is uniformly bounded as
$|\Advan_{h,M_\actorit}^{\pi_\actorit}(\state, \action) | \leq 2$ for
all $(\state, \action)$.  Then for any fixed policy $\pi$, we have
\begin{subequations}  
\begin{align}
\label{EqnActorBoundGeneral}  
\frac{1}{\actortot} \sumactor \big \{ V^{\pi}_{1,
  M_\actorit}(\state_1) - V^{\pi_\actorit}_{1,M_\actorit}(\state_1)
\big \} & \leq \horizon \left[ \frac{\log |\ActionSpace|}{\eta
    \actortot} + \eta \right].
\end{align}
In particular, setting $\eta = \sqrt{\frac{\log
    |\ActionSpace|}{\actortot}}$ yields the bound
\begin{align}
\label{EqnActorBound}  
  \frac{1}{\actortot} \sumactor \big \{
  V^{\pi}_{1,M_\actorit}(\state_1) -
  V^{\pi_\actorit}_{1,M_\actorit}(\state_1) \big \} & \leq
  \underbrace{2\horizon \sqrt{\frac{\log
        |\ActionSpace|}{\actortot}}}_{ = \Ropt(\actortot)}.
\end{align}
\end{subequations}
\end{proposition}
To be clear, the fixed comparator policy $\pi$ in the above bounds
need not be in $\PolicySet$.  This fact is important, as it allows us
to derive bounds relative to an arbitrary comparator.


\subsection{Combining the pieces}
\label{SecCombineOutline}

We are now ready to combine the pieces so as to prove
Theorem~\ref{thm:MainResult}.  For each iteration $\actorit \in
[\actortot]$, let $\pi_\actorit \defeq \pi_{\theta_\actorit}$ be the
policy chosen by the actor, and let $\MDP_\actorit =
\MDP_{\pi_\actorit}$ be the corresponding induced MDP.

Recall that Lemma~\ref{LemGoodProbability}, stated in
Section~\ref{AppProofPropCriticMain}, guarantees that the ``good''
event $\GoodEvent$ from equation~\eqref{EqnGoodEventMain} occurs with
probability at least $1 -\delta$.  Conditioned on the occurrence of
$\GoodEvent$, the bounds~\eqref{EqnCriticPiMain}
and~\eqref{EqnCriticPitilMain} ensure that for any comparator
$\widetilde \pi$, we have
\begin{align*}
V^{\widetilde \pi}_{1}(\state_1) - V^{\pi_\actorit}_{1}(\state_1) &
\leq V^{\widetilde \pi}_{1,M_\actorit}(\state_1) -
V^{\pi_\actorit}_{1,M_\actorit}(\state_1) + 2 \sumh \Big[
  \misepsilon_\hstep + \pess_\hstep \| \E_{(\MyState_\hstep,
    \Action_\hstep) \sim \widetilde \pi} \phi(\MyState_\hstep,
  \Action_\hstep)\|_{\Sigma^{-1}_\hstep} \Big] \\
& = \; V^{\widetilde \pi}_{1,M_\actorit}(\state_1) -
V^{\pi_\actorit}_{1,M_\actorit}(\state_1) + \Mis(\misepsilon) +
\Uncertain(\pitil; \pess).
\end{align*}
We now average over the iterations $\actorit \in [\actortot]$.  The
equality~\eqref{EqnCriticExact} from \Cref{LemCriticExact} ensures for
each iteration $\actorit$, the actor receives as an input a vector
$\wubar_\actorit$ such that
\begin{align}
  Q^{\pi_\actorit}_{h,M_\actorit}(\state, \action)
        \overset{Lem. \ref{LemCriticExact}}{=}
        \Qubar^{\pi_\actorit}_{h}(\state, \action) =
        \inprod{\phi_\hstep(\state, \action)}{\wubar_{hk}}.
\end{align}
Consequently, the action-value function $\Qubar^{\pi_\actorit}$ that
provided as input to the actor via $\wubar_\actorit$ is the
action-value function of $\pi_\actorit$ on the associated induced MDP
$M_\actorit$, i.e., $Q^{\pi_\actorit}_{M_\actorit}$. Applying the
bound~\eqref{EqnActorBound} from \Cref{PropActorMain} yields
$\frac{1}{\actortot} \sumactor \Big[V^{\widetilde \pi}_{1,M_\actorit}
  (s_1) - V^{\pi_\actorit}_{1,M_\actorit} (s_1) \Big] \leq
\Ropt(\actortot)$.  Combining with the prior display yields
\begin{align}
  \label{EqnNearFinalBound}
V^{\widetilde \pi}_{1}(s_1) - \frac{1}{\actortot} \sumactor
V^{\pi_\actorit}_{1}(s_1) & \leq \Ropt(\actortot) + \Mis(\misepsilon)
+ \Uncertain(\pitil; \pess).
\end{align}
Notice that the policy returned by the agent $\pi_{\Alg}$ is the
mixture policy of the policies $\pi_1,\dots,\pi_{\actortot}$ and its
value function is $V^{\pi_{\Alg}} = \frac{1}{\actortot} \sumactor
V^{\pi_\actorit}$.

Note that under the good event $\GoodEvent$, the
bound~\eqref{EqnNearFinalBound} holds for any comparator policy
$\pitilde$, which was the claim of the theorem.


\section{Discussion}
\label{SecDiscussion}

In this paper, we have developed and analyzed an actor-critic method
procedure, designed for finding near-optimal policies in the offline
setting.  The \Algname{} procedure introduces pessimism into the
critic's evaluation of a given policy's value function, thereby
ensuring that, under suitable parameter choices and assumptions, it
maintains (with high probability) a lower bound on the true value
function.  The actor then performs a form of mirror ascent so as to
maximize the value of these lower bounds.

An important feature of our method is that it introduces pessimism via
direct perturbations of the parameter vectors in a linear function
approximation scheme.  In this way, we avoid having to impose
additional model assumptions; moreover, the pessimism does \emph{not}
substantially increase the complexity of our under value/policy
classes, which allows us to provide minimax-optimal guarantees.  We
note that similar approaches have appeared before in the exploration
setting; for example, see the recent
papers~\cite{zanette2020learning,jin2021bellman,du2021bilinear}.
These methods enjoy similar advantages in terms of theoretical
guarantees, but at the expense of computational tractability.  In
contrast, the method of this paper entails solving a low-dimensional
second-order cone program, a simple class of convex programs for which
there exist many polynomial-time algorithms.  We enjoy this advantage
due to some key differences between the offline and online settings of
RL.  In the offline setting, it is possible to keep the actor's update
cleanly separated from the evaluation step of the critic, as we have
done here; this separation underlies the computational tractability.

Our work leaves open a number of interesting questions for future
work.  First, it would be interesting to provide some numerical
studies of the \Algname{}'s performance, so as to understand its
practical behavior relative to the theoretical guarantees provided
here.  Also, our analysis here has focused purely on approximation
using linear basis expansions; extension to more general function
classes is an important next step.  Finally, it will be interesting to
see to what extent these ideas can be translated to the more
challenging setting of exploration.

\subsection*{Acknowledgements}

This work was partially supported by NSF-DMS grant 2015454, NSF-IIS
grant 1909365, and NSF-FODSI grant 2023505 to MJW, a Stanford Artificial Intelligence Laboratory Toyota gift to EB, and a Office of Naval Research grant DOD-ONR-N00014-18-1-2640 to MJW.


\bibliographystyle{alpha}

\bibliography{rl}

\newpage
\appendix


\section{Additional Literature}
\label{sec:Literature}

For empirical studies on offline RL, see the
papers~\cite{laroche2019safe,jaques2019way,wu2019behavior,agarwal2020optimistic,wang2020critic,siegel2020keep,nair2020accelerating}
in addition to those presented in the main text.  Several works have
investigated offline policy learning, where concentrability
coefficients are introduced to account for the non-uniform error
propagation~\cite{munos2003error,munos2005error,antos2007fitted,antos2008learning,farahmand2010error,farahmand2016regularized,chen2019information,xie2020batch,xie2020Q,duan2021risk}. For
additional literature, see also the
papers~\cite{zhang2020variational,liao2020batch,fan2020theoretical,fu2020single,wang2019neural}. Concentrability
coefficients or density ratios also appears in the off-policy
evaluation problem, which is distinct from the policy learning problem
that we consider
here~\cite{zhang2020gendice,thomas2016data,farajtabar2018more,liu2018breaking,xie2019towards,yang2020off,nachum2019algaedice,yin2020near,yin2020asymptotically,duan2020minimax,uehara2020minimax,jiang2020minimax,kallus2019efficiently,tang2019doubly,nachum2020reinforcement,nachum2019dualdice,jiang2016doubly,uehara2020minimax,voloshin2021minimax,jiang2020minimax,hao2021bootstrapping}.


\section{Proof of \Cref{prop:Inclusion}}
\label{sec:proofofinclusion}

First, let us define an MDP class indexed by $N$; we will use this MDP
class to show that each inclusion is strict. At a high level, this MDP
class has a starting state $0$ where the agent can choose to go left
(action $-1$) or right (action $+1$); after that, it will keep going
left or right until the leftmost or rightmost terminal state is
reached. The reward is non-zero only at the terminal states.

For a fixed $N$, let the horizon be $H = N+1$ and consider the following chain MDP, where the state space is 
$$\StateSpace = \{N,-(N-1), \dots,-1,0,+1,\dots,N-1,N \}.$$ The
starting state is $0$, and there the agent can choose among two
actions ($-1$ and $+1$). In states $s \neq 0$ only one action is
available. Formally, we define
\begin{align}
\ActionSpace_s =
\begin{cases} \{ -1 \} & \text{if}\; s < 0 \\ \{ -1,+1 \}
  & \text{if}\; s = 0 \\ \{ +1 \} & \text{if}\; s > 0.
\end{cases}
\end{align} 
The reward is everywhere zero except in the terminal states $-N$ and
$+N$, for which it takes the values $-1$ and $+1$, respectively, for
the only action available there. The transition function is
deterministic, and the successor state is always $\state' = \state +
\action$ (e.g., action $+1$ in state $+2$ leads to state $+3$). In
other words, if the agent is a state $s$ with positive value, it will
move to $\state + 1$, and if $\state$ has negative value it will move
to $\state - 1$.

\subsection{Proof of part (a):
  Low Rank $\subset$ Restricted Closed}

We first prove that a low-rank MDP must satisfy the restricted
closedness assumption.  Assume the MDP is low rank. Then for any
$Q_{h+1} \in \Q_{h+1}$ and $\pi \in \Pi$, we have
\begin{align*}
  \T_\hstep^\pi Q_{h+1} & = \inprod{\phi_\hstep(\state,
    \action)}{w^R_\hstep} + \inprod{\phi_\hstep(\state,
    \action)}{\int_{\state'} \E_{\action' \sim \pi}
    Q_{h+1}(\state',\action') d \psi(\state')} \\
& = \inprod{\phi_\hstep(\state, \action)}{ w^R_\hstep +
    \int_{\state'}\E_{a'\sim \pi}Q_{h+1}(\state', \action') d
    \psi(\state')} \\
  & = \inprod{\phi_\hstep(\state, \action)}{w}
\end{align*}
for some $w \in \R^{d}$. Thus, we have $(\T_\hstep^\pi Q_{h+1}) \in
\Q_\hstep$ for all $Q_{h+1} \in \Q_{h+1}$ and $\pi \in \Pi$---i.e., if
the MDP is low rank then it satisfies the restricted closedness
condition.

In order to establish the strict inclusion, consider the MDP described
at the beginning of the proof with the following feature extractor:
\begin{align}
\phi(\state, \action) = \begin{cases}
+1 & \text{if} \; \action = +1 \\
-1 & \text{if} \; \action = -1.
\end{cases}
\end{align}
The MDP with this feature map is not low rank. For example, we must
have
\begin{align*}
  1 = \trans(N \mid N-1,+1) = \phi(N-1,+1)^\top \psi(N) = \psi(N)
\end{align*}
which implies $\psi(-N) = 0$ for $\psi$ to be a measure. However, this
means we won't be able to represent all transitions correctly, as we
would need to have
\begin{align*}
  1 = \trans( -N \mid -(N-1),-1) = \phi(-(N-1),-1)^\top \psi(-N) =
  -\psi(-N) = 0.
\end{align*}
This means the MDP is not low rank. However, we show that it still
satisfies the restricted closedness assumption. Notice that it is
enough to verify the condition in the reachable space, which is $|s|
+1 = h$ at timestep $h$. If the reward is zero it suffices to verify
that for all choices of $\theta_{h+1}$ we can find $\theta_\hstep$
such that
\begin{align}
\inprod{\phi(h-1,+1)}{\theta_{h}} & =
\inprod{\phi(h,+1)}{\theta_{h+1}} \\
\inprod{\phi(-(h-1),-1)}{\theta_{h}} & =
\inprod{\phi(-h,-1)}{\theta_{h+1}}.
\end{align}
Notice that in all cases there is only one policy available at the
successor states; for any choice of $\theta_{h+1}$, just set
$\theta_\hstep = \theta_{h+1}$. It is easy to verify that at the last
step $h=H = N+1$ the reward function is either $+1$ or $-1$, depending
on the state, and can be represented by $\theta_\hstep = +1$:
\begin{align}
\inprod{\phi(H-1,+1)}{\theta_{H}} & = +1 \\
\inprod{\phi(-(H-1),-1)}{\theta_{H}} & = -1.
\end{align}


\subsection{Proof of part (b): Restricted Closedness
  $\subset$ Linear $Q^\pi$}

We first show that every MDP that satisfies restricted closedness
satisfies the linear $Q^\pi$ assumption. For any time step $\hstep \in
\horizon$, and for a given policy $\pi \in \Pi$, if restricted
closedness holds, choose $Q_{h+1} = Q^{\pi}_{h+1}$ in the definition
of restricted closedness and use the Bellman equations to obtain
\begin{align*}
Q^{\pi}_{h} \defeq \T^{\pi}_\hstep Q^{\pi}_{h+1} \in \Q_\hstep.
\end{align*}
Thus, the linear $Q^\pi$ assumption is automatically satisfied.

In order to show the strict inclusion, consider again the MDP
described at the beginning of the proof, but with a different feature
map. The map reads
\begin{align*}
  \phi(\state, \action) = \begin{cases} [+1,0] & \text{if} \; \action
    = +1, \state \neq 0 \\
    [0,+1] & \text{if} \; \action = -1, \state \neq 0, \end{cases}
\end{align*}
and at the start state
\begin{align*}
\phi(0, \action) = \begin{cases} +1 & \text{if} \; \action = +1 \\ -1
  & \text{if} \; \action = -1.  \end{cases}
\end{align*}
Notice that we only need to verify that restricted closedness does not
hold at some timestep. When $\theta_2 = [+1, +1]$, there is no
$\theta_1$ such that
\begin{align*}
+ \theta_1 = \inprod{\phi(0,+1)}{\theta_1} & =
\inprod{\phi(1,1)}{\theta_2} = 1 \\
-\theta_1 = \inprod{\phi(0,-1)}{\theta_1} & =
\inprod{\phi(-1,-1)}{\theta_2} =1.
\end{align*}
The MDP however satisfies the linear $Q^\pi$ assumption with $\theta_1
= 1$ and $\theta_\hstep = [+1,-1]$ for $h \geq 2$.


\section{Proofs for the critic}
\label{SecCritic}

In this section, we collect together the statements and proofs of
various technical results that underlie the critic's analysis in
Section~\ref{SecCriticOutline}.  In
Section~\ref{SecProofLemCriticExact}, we prove
Lemma~\ref{LemCriticExact} that guarantees exactness of the critic on
the induced MDP, whereas Section~\ref{AppProofPropCriticMain} is
devoted to proving our main guarantee for the critic, namely
\Cref{PropCriticMain}.

Let us introduce some additional notation that plays an important role
in the proof.  Recall the regression operator $\regest_\hstep^\pi$ and
sup-norm projection operator $\SupProj^\pi_\hstep$ that were
previously defined in equations~\eqref{EqnDefnRegressionOperator}
and~\eqref{EqnSupProject}, respectively.  In addition to these two
operators, our proof also makes use of the \emph{approximation error
operator}
\begin{align}
  \label{EqnApproximationError}
\AppError{\pi}{\hstep}(F)(\state, \action) & \defeq
\inprod{\phi(\state, \action)}{\SupProj^\pi_\hstep(F)} -
\(\T^\pi_\hstep F\)(\state, \action),
\end{align}
which is a mapping from $\Fclass$ to itself.


\subsection{Proof of Lemma~\ref{LemCriticExact}}
\label{SecProofLemCriticExact}

By definition, the induced MDP differs from the original MDP only by
the perturbation of the reward function.  Thus, by definition of value
functions, we can write
\begin{subequations}  
  \begin{align}
\label{EqnRepOne}
Q^\pi_{\hstep, \IndMDP{\pi}}(\state, \action) - Q^\pi_{\hstep}(\state,
\action) & = \sum_{\ell=h}^\horizon \E_{(\MyState_\ell, \Action_\ell)
  \sim \pi \mid (\state, \action)} \left[
  \rewardhat^\pi_\hstep(\MyState_\ell, \Action_\ell) -
  \reward_\hstep(\MyState_\ell, \Action_\ell) \right].
  \end{align}
On the other hand, using the definition of $\Qubar^\pi_\hstep$ and the
Bellman conditions, we have
\begin{align*}
\Qubar_\hstep^\pi(\state, \action) - Q_\hstep^\pi(\state, \action) & =
\inprod{\phi(\state, \action)}{\wubar_\hstep^\pi} - \T^\pi_\hstep
(Q^\pi_{h+1})(\state, \action) \\
& = \left \{ \inprod{\phi(\state, \action)}{\wubar_\hstep^\pi} -
\T_\hstep^\pi( \Qubar^\pi_{h+1})(\state, \action) \right \} + \left \{
\T_\hstep^\pi(\Qubar^\pi_{h+1})(\state, \action) +
\T^\pi_\hstep(Q^\pi_{h+1})(\state, \action) \right \} \\
& = \rewardhat^\pi_\hstep(\state, \action) - \reward_\hstep(\state,
\action) + \Es{\trans_\hstep(\state, \action)} \Ea{\pi(\cdot \mid
  \MyState')} (\Qubar_{h+1}^\pi - Q_{h+1}^\pi)(\MyState', \Action')
\end{align*}
Applying this argument recursively to $\ell = h+1, \ldots, \horizon$,
we find that
\begin{align}
  \label{EqnRepTwo}
\Qubar_\hstep^\pi(\state, \action) - Q_\hstep^\pi(\state, \action) & =
\sum_{\ell = h}^\horizon \E_{(\MyState_\ell, \Action_\ell) \sim \pi
  \mid (\state, \action)} \Big [ \rewardhat^\pi_\hstep(\MyState_\ell,
  \Action_\ell) - \reward_\hstep(\MyState_\ell, \Action_\ell) \Big]
  \end{align}
\end{subequations}
Subtracting equation~\eqref{EqnRepTwo} from equation~\eqref{EqnRepOne}
yields the claim.


\subsection{Proof of Proposition~\ref{PropCriticMain}}
\label{AppProofPropCriticMain}

We split the proof into two parts, corresponding to the two bounds.

\subsubsection{Proof of the bound~\eqref{EqnCriticPiMain}}

We begin by proving the bound on the critic's estimate for
the value function of the input policy $\pi$.

\paragraph{High-level roadmap:}  We begin by outlining the main
steps in the proof.  Our first step is to define a sequence of weight
vectors $\what \defeq \{\what^\pi_\hstep \}_{\hstep=1}^\horizon$ such
that
\begin{subequations}
\begin{align}
\label{EqnClaimOne}
\Big| \sum_{\action_1 \in \ActionSpace} \pi(\action_1 \mid \state_1)
\inprod{\phi_1(\state_1, \action_1)}{\what^\pi_1} - V^\pi_1(\state_1)
\Big| & \leq \sumh \misepsilon_h.
\end{align}
Our second step is to show that conditioned on the good event
$\GoodEvent(\pess)$ from equation~\eqref{EqnGoodEventMain}, the
sequence $\what$ is feasible for the critic's convex program; this
feasibility, combined with the optimality of $\wunder$, implies that
\begin{align}
\label{EqnClaimTwo}
V^\pi_{1, \IndMDP{\pi}}(\state_1) & \stackrel{(i)}{=} \sum_{\action_1
  \in \ActionSpace} \pi(\action_1 \mid \state_1)
\inprod{\phi_1(\state_1, \action_1)}{\wunder_1^\pi} \; \leq \;
\sum_{\action_1 \in \ActionSpace} \pi(\action_1 \mid \state_1)
\inprod{\phi_1(\state_1, \action_1)}{\what^\pi_1}.
\end{align}
\end{subequations}
Here step (i) follows from Lemma~\ref{LemCriticExact}, which
guarantees that the estimated value functions $\Vubar^\pi_\hstep$ of
the critic are exact in the induced MDP.  Combining the two
bounds~\eqref{EqnClaimOne} and~\eqref{EqnClaimTwo} yields $V^\pi_{1,
  \IndMDP{\pi}}(\state_1) \leq V^\pi_1(\state_1) + \sumh
\misepsilon_h$, as claimed in equation~\eqref{EqnCriticPiMain}. \\

\noindent It remains to prove our two auxiliary
claims~\eqref{EqnClaimOne} and~\eqref{EqnClaimTwo}.

\paragraph{Proof of claim~\eqref{EqnClaimOne}:}

Given a policy $\pi$, we use backwards induction to define the
sequence $\{\what^\pi \}_{\hstep=1}^\horizon$ by first setting
\mbox{$\what^\pi_{\horizon + 1} = 0$,} and then defining
\begin{align}
  \label{EqnDefnWhat}
    \what_\hstep^\pi & \defeq
    \SupProj^\pi_\hstep(\Qhat^\pi_{\hstep+1}) \qquad \mbox{for $\hstep
      = \horizon, \horizon-1, \ldots, 1$,}
\end{align}
where $\Qhat^\pi_{\hstep +1}(\state, \action) \defeq
\inprod{\phi_{\hstep+1}(\state, \action)}{\what^\pi_{\hstep+1}}$.  By
construction, we have the bound \mbox{$\| \what^\pi_\hstep \|_2 \leq
  \rad_\hstep$} for all \mbox{$h \in [\horizon]$.}  The following
lemma bounds the sup-norm distance between the induced linear
$Q$-value function estimate, and the actual $Q^\pi$-value function.

\begin{lemma}
\label{lem:BestPredictorAccuracy}
The functions $\{\Qhat^\pi_{\hstep} \}_{\hstep=1}^\horizon$ defined by
the best-predictor sequence $\{ \what^\pi_\hstep
\}_{\hstep=1}^{\horizon}$ from equation~\eqref{EqnDefnWhat} satisfy
the bound
\begin{align}
  \label{EqnBestSequenceApprox}
\big| \Qhat^\pi_\hstep(\state, \action) - Q_\hstep^\pi(\state,
\action)\big| & \leq \sum_{\ell=h}^\horizon \misepsilon_\ell \qquad
\mbox{for all $\hstep \in [\horizon]$.}
\end{align}
\end{lemma}
\begin{proof}
Introduce the shorthand \mbox{$\Delta_\hstep(\state, \action) \defeq
  \Qhat^\pi_\hstep(\state, \action) - Q_\hstep^\pi(\state, \action)$}
for the error at stage $\hstep$ to be bounded.  Since
\mbox{$Q_\hstep^\pi = \T^\pi_\hstep(Q^\pi_{h+1})$,} we can write
\begin{align*}
  \Delta_\hstep(\state, \action) & = \Qhat^\pi_\hstep(\state, \action)
  - Q^\pi_\hstep(\state, \action) \\
& = \Qhat^\pi_\hstep(\state, \action) - (\T^\pi_\hstep
  \Qhat^\pi_{h+1})(\state, \action) + (\T^\pi_\hstep
  \Qhat^\pi_{h+1})(\state, \action) -
  \T^\pi_\hstep(Q_{\hstep+1}^\pi)(\state, \action) \\
& = \Qhat^\pi_\hstep(\state, \action) - (\T^\pi_\hstep
  \Qhat^\pi_{h+1})(\state, \action) + \Es{\trans_\hstep(\state,
    \action)} \Ea{\pi(\cdot \mid \MyState')} \Big [
    \Qhat^\pi_{\hstep+1}(\MyState', \Action') - Q^\pi_{h+1}(\MyState',
    \Action') \Big] \\
& = \sum_{\ell=h}^\horizon \E_{(\MyState_\ell, \Action_\ell) \sim \pi
    \mid (\state, \action)} \Big[\Qhat_\ell^\pi(\MyState_\ell,
    \Action_\ell) - \T^\pi_\ell(\Qhat^\pi_{\ell+1})(\MyState_\ell,
    \Action_\ell) \Big],
\end{align*}
where the final equality follows by induction.

From the definition~\eqref{EqnDefnWhat} of $\what$ and the function
estimate $\Qhat^\pi_\ell(\state, \action) = \inprod{\phi_\ell(\state,
  \action)}{\what^\pi_\ell}$, combined with the Bellman approximation
condition, we have
\begin{align*}
  \Big|\Qhat_\ell^\pi(\state, \action) - (\T^\pi_\ell
  \Qhat^\pi_{\ell+1})(\state, \action) \Big| & \leq
  \AppError{\pi}{\ell}(\Qhat^\pi_{\ell+1}) \; \leq \;
  \misepsilon_\ell,
\end{align*}
uniformly over all $\ell$, and over all state-action pairs $(\state,
\action)$.  Summing these bounds completes the proof.
\end{proof}

\paragraph{Proof of claim~\eqref{EqnClaimTwo}:}
In order to prove this claim, we need to exhibit a sequence $\xi =
(\xihat_1, \ldots, \xihat_\horizon)$ such that the pair $(\xihat,
\what)$ are feasible for the critic's convex
program~\eqref{EqnCriticProgram}.  In particular, we need to ensure
the following three conditions:
\begin{enumerate}
\item[(a)] $\|\what^\pi_\hstep\|_2 \leq \rad_\hstep$
  for all $\hstep \in [\horizon]$
\item[(b)] $\Signorm{\xihat_h} \leq \pess_h$ for all $h \in
  [\horizon]$.
\item[(c)] We have $\what^\pi_\hstep = \xihat^\pi_\hstep +
  \regest^\pi_\hstep(\Qhat^\pi_{\hstep+1})$ for all $h \in [\hstep]$.
\end{enumerate}
Note that condition (a) is automatically satisfied by the
definition~\eqref{EqnDefnWhat} of $\what$, since the projection
$\SupProj^\pi_{\hstep}$ imposes this Euclidean norm bound.

It remains to exhibit a choice of $\xihat$ such that conditions (b)
and (c) hold.  Since $\what^\pi_{\hstep} =
\SupProj^\pi_{\hstep}(\Qhat^\pi_{\hstep})$ by definition, condition (c)
forces us to set
\begin{align*}
  \xihat^\pi_\hstep = \SupProj^\pi_{\hstep}(\Qhat^\pi_{\hstep+1}) -
  \regest^\pi_\hstep(\Qhat^\pi_{\hstep+1}) \; = \;
-  \TotError{\pi}{\hstep}(\Qhat^\pi_{\hstep+1}).
\end{align*}
But since the event $\GoodEvent(\pess)$ holds by assumption, we have
\begin{align*}
\Signorm{\xihat^\pi_\hstep} & =
\Signorm{\TotError{\pi}{\hstep}(\Qhat^\pi_{\hstep+1})} \leq \pess_h,
\end{align*}
showing that this choice of $\xihat$ satisfies condition (b).


\subsubsection{Proof of part (b)}

Here we prove the bound~\eqref{EqnCriticPitilMain} stated in part (b)
of the lemma, which provides an inequality on the value function error
for an arbitrary policy.

Our proof is based on establishing an auxiliary result that implies
the claim.  In particular, we first show that for any policy $\pitil$,
we have
\begin{align}
\label{EqnSubClaim}  
\Big|V^{\pitil}_{1,\IndMDP{\pi}}(s_1) - V^{\pitil}_{1}(s_1) \Big| \leq
\sumh \Signorminv{\MeanVec{\pitil}_\hstep} \left \{ \pess_\hstep +
\Signorm{\TotError{\pi}{\hstep}(\Qubar^\pi_{\hstep +1})} \right \} +
\sumh \misepsilon_\hstep,
\end{align}
where $\MeanVec{\pitil}_\hstep \defeq \E_{( \MyState_\hstep,
  \Action_\hstep) \sim \pitil} [\phi(\MyState_\hstep,
  \Action_\hstep)]$.  Since
$\Signorm{\TotError{\pi}{\hstep}(\Qubar^\pi_{\hstep +1})} \leq
\pess_\hstep$ conditioned on $\GoodEvent(\pess)$, this implies the
claim.

Let us now prove the auxiliary claim~\eqref{EqnSubClaim}.  First, we
observe that by definition, the perturbation in the reward can be
written as
\begin{align*}
\rewardhat^\pi_\hstep(\state, \action) - \reward_\hstep(\state,
\action) & \stackrel{(i)}{=}
\inprod{\phi_h(\state,\action)}{\wunder_\hstep^\pi} -
\T^\pi_h(\Qubar_{\hstep+1}^\pi)(\state, \action) \\
& \stackrel{(ii)}{=}
\inprod{\phi_h(\state,\action)}{\xiunder_\hstep^\pi} +
\inprod{\phi_h(\state,
  \action)}{\regest^\pi_\hstep(\Qubar^\pi_{\hstep+1})} -
\T^\pi_h(\Qubar_{\hstep+1}^\pi)(\state, \action) \\
& \stackrel{(iii)}{=}
\inprod{\phi_h(\state,\action)}{\xiunder_\hstep^\pi} +
\inprod{\phi_h(\state,
  \action)}{\TotError{\pi}{\hstep}(\Qubar^\pi_{\hstep+1})} +
\AppError{\pi}{\hstep}(\Qubar^\pi_{\hstep+1})(\state, \action),
\end{align*}
where step (i) uses the definition $\Qubar^\pi_\hstep(\state, \action)
= \inprod{\phi_h(\state, \action)}{\wubar^\pi_\hstep}$; step (ii) uses
the relation \mbox{$\wubar^\pi_\hstep = \xiunder^\pi_\hstep +
  \regest^\pi_\hstep(\Qubar^\pi_{\hstep +1 })$;} and step (iii)
involves adding and subtracting $\inprod{\phi_h(\state,
  \action)}{\SupProj^\pi_\hstep(\Qubar^\pi_{\hstep+1})}$, and using
the definitions of the approximation
error~\eqref{EqnApproximationError} and the error
operator~\eqref{EqnDefnParError}.

Since the induced MDP differs from the original only by the reward
perturbation, we have
\begin{align*}
 \Big|V^{\pitil}_{1,\IndMDP{\pi}}(s_1) - V^{\pitil}_{1}(s_1)\Big| & =
 \Big|\sumh \E_{(\MyState_\hstep, \Action_\hstep)\sim \pitil}
 \Big[\rewardhat^\pi_\hstep(\MyState_\hstep, \Action_\hstep) -
   \reward_\hstep(\MyState_\hstep, \Action_\hstep) \Big] \Big| \\
 & = \Big| \sumh \E_{(\MyState_\hstep, \Action_\hstep)\sim \pitil}
 \Big[ \inprod{\phi_h(\MyState_\hstep,
     \Action_\hstep)}{\xiunder_\hstep^\pi +
     \TotError{\pi}{\hstep}(\Qubar^\pi_{\hstep+1})} +
   \AppError{\pi}{\hstep}(\Qubar^\pi_{\hstep+1})(\MyState_\hstep,
   \Action_\hstep) \Big] \Big|.
\end{align*}
We now observe that
$|\AppError{\pi}{\hstep}(\Qubar^\pi_{\hstep+1})(\MyState_\hstep,
\Action_\hstep)| \leq \misepsilon_h$ by the Bellman closure
assumption.  As for the first term, introducing the shorthand
$\MeanVec{\pitil}_\hstep \defeq \E_{(\MyState_\hstep,
  \Action_\hstep)\sim \pitil} \big[\phi_h(\MyState_\hstep,
  \Action_\hstep) \big]$, we have
\begin{align*}
\E_{(\MyState_\hstep, \Action_\hstep)\sim \pitil} \Big[
  \inprod{\phi_h(\MyState_\hstep, \Action_\hstep)}{\xiunder_\hstep^\pi
    + \TotError{\pi}{\hstep}(\Qubar^\pi_{\hstep+1})} \Big] & \leq
\Signorminv{ \MeanVec{\pitil}_\hstep} \; \Signorm{ \xiunder_\hstep^\pi
  + \TotError{\pi}{\hstep}(\Qubar^\pi_{\hstep+1})} \\
& \leq \Signorminv{ \MeanVec{\pitil}_\hstep} \Big \{ \pess_h +
\Signorm{ \TotError{\pi}{\hstep}(\Qubar^\pi_{\hstep+1})} \Big \},
\end{align*}
where the final step combines the triangle inequality, with the fact
that $\Signorm{\xiunder_\hstep^\pi} \leq \pess_\hstep$, since
$\xiunder_\hstep^\pi$ must be feasible for the critic's convex
program~\eqref{EqnCriticProgram}.  Putting together the pieces yields
the claim~\eqref{EqnSubClaim}.


\subsection{Proof of Lemma~\ref{LemGoodProbability}}
\label{SecProofLemGoodProbability}

We now prove Lemma~\ref{LemGoodProbability}, which asserts that the
good event $\GoodEvent(\delta)$, as defined in
equation~\eqref{EqnGoodEventMain}, holds with high probability when
the pessimism parameters are chosen according to
equation~\eqref{EqnPessChoice}.

Recall from equation~\eqref{EqnDefnParError} that for any pair $(Q,
\pi)$, the associated parameter error is given by the difference
\mbox{$\TotError{\pi}{\hstep}(Q) =\regest^\pi_\hstep(Q) -
  \SupProj^\pi_\hstep(Q)$.}  We begin with a simple lemma that
decomposes this error into three terms.  In order to state the lemma,
we introduce two forms of error variables: statistical and
approximation-theoretic.

Recall that $\Index_\hstep$ denotes the subset of indices associated
with time step $\hstep$.  The first noise variables take the form
\begin{subequations}
\label{EqnErrorSubterms}
\begin{align}
\label{EqnEtaNoise}
\eta_{hk}(Q,\pi) & \defeq r_{hk} + \Ea{\pi(\cdot\mid s_{hk})}
Q(\state_{h+1,k}, \Action') - (\T^\pi_\hstep Q)(s_{hk},a_{hk}),
  \end{align}
defined for each $\hstep \in [\horizon]$ and $k \in \Index_\hstep$.
Note that conditionally on the pair $(\state_{hk}, \action_{hk})$, our
sampling model and the definition of the Bellman operator
$\T^\pi_\hstep$ ensures that each $\eta_{hk}$ is zero-mean random
variable, corresponding to a form of statistical error.  Our analysis
also involves some approximation error terms, in particular via the
quantities
\begin{align}
\label{EqnDeltaNoise}    
\Delta_{hk}(Q,\pi) & \defeq - \AppError{\pi}{\hstep}(Q)(\state_{hk},
\action_{hk}) \; = \; (\T^\pi_\hstep Q)(\state_{hk}, \action_{hk}) -
\inprod{\phi_\hstep(\state_{hk},
  \action_{hk})}{\SupProj^\pi_\hstep(Q)}
\end{align}
\end{subequations}
With these definitions, we have the following guarantee:
\begin{lemma}[Decomposition of $\TotError{\pi}{\hstep}(Q)$]
\label{LemParErrorDecomp}
For any pair $(Q, \pi)$, we have the decomposition
\begin{align}
\label{EqnParErrorDecomp}
\TotError{\pi}{\hstep}(Q) & = e^\eta_\hstep(Q,\pi) +
e^\lambda_\hstep(Q,\pi) + e^\Delta_\hstep(Q,\pi),
\end{align}
where the three error terms are given by
\begin{subequations}
\begin{align}
\label{EqnNoiseError}
e^\eta_\hstep(Q,\pi) & \defeq \Sigma^{-1}_\hstep \sum_{k \in
  \Index_\hstep} \phi_{hk} \eta_{hk}(Q,\pi), & \text{(Statistical
  estimation error)} \\
\label{EqnRegError}
e^\lambda_\hstep(Q,\pi) & \defeq -\lambda \Sigma^{-1}_\hstep
\SupProj^\pi_\hstep(Q), & \text{(Regularization error)}, \quad
\mbox{and} \\
\label{EqnClosureError}
e^\Delta_\hstep(Q,\pi) & \defeq \Sigma^{-1}_\hstep \sum_{k \in
  \Index_\hstep} \phi_{hk} \Delta_{hk}(Q; \pi) & \text{(Approximation
  error)}.
\end{align}
\end{subequations}
\end{lemma}
\noindent See
Section~\ref{SecProofLemParErrorDecomp} for the proof of this claim.

The remainder of our analysis is focused on bounding these three
terms.  Analysis of the regularization error and approximation error
terms is straightforward, whereas bounding the statistical estimation
error requires more technical effort.  We begin with the two easy
terms.

\paragraph{Regularization error:}

Beginning with the definition~\eqref{EqnRegError}, we have
\begin{align}
\label{EqnRegBound}
\Signorm{e_\hstep^\lambda(Q,\pi)} & = \lambda \Signorminv{
  \SupProj^\pi_\hstep(Q)} \stackrel{(i)}{\leq} \sqrt{\lambda}
\|\SupProj^\pi_\hstep(Q) \|_2 \stackrel{(ii)}{\leq} \sqrt{\lambda},
\end{align}
where step (i) follows since $\Sigma_h \succeq \lambda I$; and
inequality (ii) follows from the bound
\mbox{$\|\SupProj^\pi_\hstep(Q)\|_2 \leq \rad_\hstep \leq 1$},
guaranteed by the definition of $\SupProj^\pi_\hstep$.


\paragraph{Approximation error:}
By definition, we have $\Signorm{e^\Delta_\hstep(Q,\pi)} =
\Signorminv{ \sum_{k \in \Index_\hstep} \phi_{hk} \Delta_{hk}(Q,
  \pi)}$.  By the Bellman approximation condition, we have
$|\Delta_{hk}(Q, \pi)| \leq \misepsilon_h$ uniformly over all $k$.
Consequently, applying Lemma 8 (Projection Bound) from the
paper~\cite{zanette2020learning} guarantees that
\begin{align}
\label{EqnApproxBound}
\Signorm{e^\Delta_\hstep(Q,\pi)} & \leq \sqrt{\Nsamp_\hstep}
\misepsilon_\hstep.
\end{align}

\paragraph{Statistical estimation error:}

Lastly, we turn to the analysis of the statistical estimation error.
In particular, we prove the following guarantee:

\begin{lemma}
  \label{LemStatisticalError}
There is a universal constant $c > 0$ such that
\begin{multline}
\label{EqnStatBound}  
\Signorm{e_\hstep^\eta(Q,\pi)}^2 \leq c \left \{1 + d_\hstep \log
\big( 1 + \tfrac{\numepoch}{d_\hstep \lambda} \big) + d_\hstep\log \big(1 + 8
\sqrt{\numepoch} \big) + d_\hstep \log \(1 + 16 R\sqrt{T} \) +
\log\frac{\horizon}{\delta} \right \}
\end{multline}
uniformly over all $Q \in \Qlin_\hstep$, $\pi \in \PiSoft(R)$ and
$\hstep \in [\horizon]$ with probability at least $1 - \delta$.
\end{lemma}
\noindent See Section~\ref{SecProofLemStatisticalError} for the proof
of this claim.

\paragraph{Putting together the pieces:}  By combining our three
bounds---namely, equations~\eqref{EqnRegBound},
~\eqref{EqnApproxBound} and~\eqref{EqnStatBound}, we conclude that
with the choice
\begin{multline*}
  \pess_h(\delta) \defeq \sqrt{\lambda} + \sqrt{\Nsamp_\hstep}
  \misepsilon_\hstep + \\
  c \left \{1 + d_\hstep \log \big( 1 +
  \tfrac{\numepoch}{d_\hstep \lambda} \big) + d_\hstep\log \big(1 + 8
  \sqrt{\numepoch} \big) + d_\hstep \log \(1 + 16 R\sqrt{T} \) +
  \log\frac{\horizon}{\delta} \right \}^{1/2},
\end{multline*}
the good event $\GoodEvent(\delta)$ holds with probability at least
$1-\delta$.  This completes the proof of
Lemma~\ref{LemGoodProbability}.

It remains to prove the two auxiliary lemmas that we stated: namely,
Lemma~\ref{LemParErrorDecomp} that gave a decomposition of the
parameter error, and Lemma~\ref{LemStatisticalError} that bounded the
statistical error.  We do so in
Sections~\ref{SecProofLemParErrorDecomp}
and~\ref{SecProofLemStatisticalError}, respectively.
  

\subsubsection{Proof of Lemma~\ref{LemParErrorDecomp}}
\label{SecProofLemParErrorDecomp}

Starting with the definition~\eqref{EqnDefnRegressionOperator} of the
regression operator $\regest^\pi_\hstep$, we have
\begin{align*}
  \regest^\pi_\hstep(Q) & \defeq \Sigma^{-1}_\hstep \sum_{k \in
    \Index_\hstep} \phi_{hk} [r_{hk} + \Ea{\pi(\cdot \mid s_{hk})}
    Q(\state_{h+1,k}, \Action')] \\
& \stackrel{(i)}{=} \Sigma^{-1}_\hstep \sum_{k \in \Index_\hstep}
  \phi_{hk} [(\T^\pi_\hstep Q)(\state_{hk}, \action_{hk})] +
  \underbrace{\Sigma^{-1}_\hstep \sum_{k \in \Index_\hstep} \phi_{hk}
    \eta_{hk}(Q, \pi)}_{ = e^\eta_\hstep(Q, \pi)}
\end{align*}
where equality (i) follows by adding and subtracting terms, and using
the definition~\eqref{EqnEtaNoise} of $\eta_{hk}$.

Next we use the definition~\eqref{EqnDeltaNoise} of the approximation
error terms $\Delta_{hk}$ to find that
\begin{align*}
  \regest^\pi_\hstep(Q) & = \xi_\hstep + \Sigma^{-1}_\hstep \( \sum_{k
    \in \Index_\hstep} \phi_{hk} \big[
    \inprod{\phi_{hk}}{\SupProj^\pi_\hstep(Q)} + \Delta_{hk}(Q,\pi)
    \big] \) + e^\eta_\hstep(Q, \pi)
\end{align*}
Since $\Sigma_\hstep = \sum_{k \in \Index_\hstep} \phi_{hk}
\phi_{hk}^\top + \lambda I$, we can write
\begin{align*}
  w_\hstep(Q,\pi, \xi_\hstep) & = \xi_\hstep + \Sigma^{-1}_\hstep \Big
  \{ \Sigma_\hstep \wstar_\hstep(Q, \pi) + \sum_{k \in \Index_\hstep}
  \phi_{hk} \Delta_{hk}(Q,\pi) - \lambda \wstar_\hstep(Q, \pi) \Big \}
  + e^\eta_\hstep \\
& = \xi_\hstep + \wstar_\hstep(Q,\pi) +\Sigma^{-1}_\hstep \( \sumk
  \phi_{hk} \Delta_{hk}(Q,\pi) - \lambda \wstar_\hstep(Q,\pi)\) +
  e^\eta_\hstep \\
& = \xi_\hstep + \wstar_\hstep(Q,\pi) + e^\eta_\hstep +
  e^\lambda_\hstep + e^\Delta_\hstep,
\end{align*}
which completes the proof.


\subsubsection{Proof of Lemma~\ref{LemStatisticalError}}
\label{SecProofLemStatisticalError}

From the definition~\eqref{EqnDefnQclass}, we need to study the
constrained class of linear action-value functions based on radii
$\rad_\hstep \in (0,1]$ for all $\hstep \in [\horizon]$.  As for the
  constraint defining the soft-max policy
  class~\eqref{EqnDefnSoftMax}, let us upper bound how large the
  $\ell_2$-norm of the actor's parameter vector can be over
  $\actortot$ iterations.
  
Based on the actor's updates, we have the bound
\begin{align*}
  \| \theta_{\actorit,\hstep} \|_2 = \| \sum_{\actorit = 1}^\actortot
  \eta w_{\actorit,\hstep} \|_2 \leq \eta \sum_{\actorit =
    1}^\actortot \| w_{\actorit,\hstep} \|_2 \stackrel{(i)}{\leq} \eta
  \actortot \rad_\hstep \stackrel{(ii)}{\leq} \eta \actortot,
\end{align*}
where step (i) follows from the definition of the critic's
program~\eqref{EqnCriticProgram}, and step (ii) follows from the
assumption $\rad_\hstep \in (0, 1]$.  Thus, we are assured that $R =
  \eta \actortot$ is an upper bound on this $\ell_2$-norm.

We make use of a discretization argument to control the associated
empirical process.  Let $\Ncover_\infty(\epsilon; \Q)$ denote the
cardinality of the smallest $\epsilon$-covering of $\Q$ in the
sup-norm---that is, a collection $\{Q^i\}_{i=1}^N$ such that for all
$Q \in \Q$, we can find some $i \in [N]$ such that
\begin{align*}
\|Q - Q^i\|_\infty = \sup_{(\state, \action)} |Q(\state, \action) -
Q^i(\state, \action)| \leq \epsilon.
\end{align*}
Similarly, we let \mbox{$\Ncover_{\infty,1}(\epsilon; \Pi(R))$} denote
an \mbox{$\epsilon$-cover} of $\Pi(R)$ when measuring distances with
the norm
  \begin{align}
\label{EqnPiNorm}    
\|\pi - \pi'\|_{\infty, 1} & \defeq \sup_{\state} \sum_{\action
  \in \ActionSpace} \big| \pi(\action \mid \state) - \pi'(\action \mid
\state) \big|.
  \end{align}

We have the following bounds on these covering numbers:
\begin{lemma}[Covering number bounds]
\label{LemCovering}  
For any $\epsilon \in (0,1)$, we have
\begin{subequations}
  \begin{align}
    \label{EqnQcover}
  \log \Ncover_\infty(\epsilon; \Q) & \leq d \log \(1 +
  \tfrac{2}{\epsilon} \) \qquad \mbox{and} \\
  \label{EqnPiCover}
  \log \Ncover_{\infty, 1} \( \epsilon; \Pi(R) \) & \leq d \log \(1 +
  \tfrac{16 R}{\epsilon} \).
\end{align}
\end{subequations}
\end{lemma}
\noindent See Section~\ref{SecProofLemCovering} for the proofs of
these claims.

For any $\epspi \in (0,1)$, we define
\begin{align}
\beta(\epspi) & \defeq d \log \big( 1 + \tfrac{\numepoch}{d \lambda}
\big) + \log \Ncover_{\infty}(\epsq; \Q) + \log \Ncover_{\infty,
  1}(\epspi; \PiSoft) + \log\frac{\horizon}{\delta}
\end{align}

Given this definition and the bounds from Lemma~\ref{LemCovering}, the
proof of Lemma~\ref{LemStatisticalError} is reduced to showing that
for any $\epsilon \in (0,1)$, there is a universal constant $c$ such
that
\begin{align}
\label{EqnSupremum}  
\max_{\hstep \in [\horizon]} \sup_{ \substack{Q \in \Qlin_\hstep
    \\ \pi \in \PiSoft}} \Signorm{e_\hstep^\eta(Q,\pi)} \leq c \,
\sqrt{\beta(\epspi)} + 4 \sqrt{\numepoch} \epsq
\end{align}
with probability at least $1 - \delta$.  The claim stated in
Lemma~\ref{LemStatisticalError} follows from the choice $\epsilon =
\tfrac{1}{4 \sqrt{\numepoch}}$.  The remainder of our proof is devoted
to the proof of this claim.


\paragraph{Proof of the claim~\eqref{EqnSupremum}:}

Let us recall the definition
\begin{align*}\eta_{hk}(Q, \pi) = r_{hk} + \Ea{\pi_\hstep(\cdot\mid
  \state_{hk}) } Q(\state_{h+1,k}, \Action') - (\T^\pi_\hstep
  Q)(\state_{hk}, \action_{hk}).
\end{align*}
Consequently, by starting with the definition of $e_\hstep^\eta$ and
applying the triangle inequality, we obtain the upper bound
$\Signorm{e_\hstep^\eta(Q,\pi)} = \Signorminv{ \sumk \phi_{hk}
  \eta_{hk}(Q, \pi)} \leq \Term_1 + \Term_2(Q, \pi)$, where
\begin{align*}
\Term_1 & \defeq \Signorminv{\sumk \phi_{hk} [\underbrace{r_{hk} -
      r(s_{hk},a_{hk})}_{\defeq Y_{hk}}]} \; \; \mbox{and} \\
\Term_2(Q, \pi) & \defeq \Big\| \sumk \phi_{hk} [Q(s_{h+1,k},\pi) -
  \Es{\trans(\cdot \mid \state_{hk}, \action_{hk})} Q(\MyState',\pi)]
\Big\|_{\Sigma_\hstep^{-1}}
\end{align*}
For a fixed $(\pi,Q)$ and conditioned on the sampling history, both $\Term_1$ and $\Term_2$ are mean zero.
Note that $\Term_1$ is independent of the pair $(Q, \pi)$, so that its
analysis does not require discretization techniques.  On the other
hand, analyzing $\Term_2(Q, \pi)$ does require a reduction step via
discretization, with which we begin.

Introducing the shorthand $N = \Ncover(\epsq, \Q)$, let $\{Q^i
\}_{i=1}^N$ be an $\epsq$-cover of the set $\Q$ in the sup-norm.
Similarly, with the shorthand $J = \Ncover(\epspi, \Pi)$, let
$\{\pi^j\}_{j=1}^J$ be an $\epspi$-cover of $\Pi$ in the
norm~\eqref{EqnPiNorm}.  For a given $Q$, let $Q^i$ denote the member
of the cover such that $\|Q - Q^i\|_\infty \leq \epsq$.  With this
choice, we have
\begin{align*}
\Term_2(Q, \pi) & = \Term_2(Q^i, \pi) + \{ \Term_2(Q, \pi) -
\Term_2(Q^i, \pi) \}.
\end{align*}
Similarly, let $\pi^m$ be a member of the cover such that $\|\pi(\cdot
\mid \state) - \pi^m(\cdot \mid \state)\|_1 \leq \epspi$ for all $\state$.
With this choice, we have
\begin{align*}
\Term_2(Q, \pi) & \leq \Term_2(Q^i, \pi^m) + \underbrace{\{
  \Term_2(Q^i, \pi) - \Term_2(Q^i, \pi^m) \}}_{D^\pi} + \underbrace{\{
  \Term_2(Q, \pi) - \Term_2(Q^i, \pi) \}}_{D^Q}.
\end{align*}
We begin by bounding the two discretization errors.  By the triangle
inequality, we have
\begin{align*}
D^Q & \leq \Big\| \sumk \phi_{hk}[\underbrace{Q(s_{h+1,k},\pi) -
    Q^i(s_{h+1,k},\pi) + \Es{p(s_{hk},a_{hk})}(Q(\MyState',\pi) -
    Q^i(\MyState',\pi))}_{\defeq E^i_{hk}(Q, \pi)}]
\Big\|_{\Sigma_\hstep^{-1}}.
\end{align*}
Our choice of discretization ensures that $|E^i_{hk}(Q, \pi)| \leq 2
\epsq$ uniformly for all $(h,k)$ and $(Q, \pi)$.  Applying Lemma 8
(Projection Bound) from the paper~\cite{zanette2020learning} ensures
that \mbox{$D^Q \leq 2 \epsq \sqrt{\numepoch}$.}  To be clear, this is
a deterministic claim; it holds uniformly over the choices of $Q$,
$Q^i$, and $\pi$.  A similar argument yields that $D^\pi \leq 2 \epspi
\sqrt{\numepoch}$.

Putting togther the pieces yields that for any $(Q, \pi)$,
we have the bound
\begin{align}
  \label{EqnKeyDiscretization}
\Term_2(Q, \pi) & \leq \max_{ \substack{i \in [N] \\ j \in [M]}}
\Term_2(Q^i, \pi^j) + 4 \sqrt{\numepoch} \epsq.
\end{align}

We now need to bound $\Term_1$ along with $\Term_2(Q^i, \pi^j)$ for a
fixed pair $(Q^i, \pi^j)$.  In order to do so, we apply known
self-normalized tail bounds~\cite{PenLaiSha09}, which apply to sums of
the form $\| \sumk \phi_{hk} V_{hk} \|_{\Sigma_\hstep^{-1}}$, where
the $V_{hk}$ form a martingale difference sequence with conditionally
sub-Gaussian tails.  Note that $\Term_1$ is of this general form with
$V_{hk} = Y_{hk}$, which is a $1$-sub-Gaussian variable by assumption.
On the other hand, the variable $\Term_2(Q^i, \pi^j)$ is of this form
with
\begin{align*}
  V_{hk} & = Q^i(s_{h+1,k},\pi^j) - \Es{p(s_{hk},a_{hk})}
  Q^i(\MyState',\pi^j).
\end{align*}
Since $|V_{hk}| \leq 1$ due to the uniform boundedness of $Q^i$,
this is a $1$-sub-Gaussian variable as well.

Consequently, Theorem 1 from the paper~\cite{Abbasi11} ensures that
\begin{align*}
\Pro \( \max \{ \Term_1, \Term_2(Q^i, \pi^j) \} \geq \log\frac{\det \Sigma_\hstep}{\det \lambda I} +
2\log \frac{1}{\delta} \) & \leq \delta.
\end{align*}
Note that $\det \lambda I = \lambda^{d_\hstep}$.  Moreover, Lemma 10
(Determinant-Trace Inequality) in \cite{Abbasi11} yields $\log \det
\Sigma_\hstep \leq d_\hstep \log \( \lambda +
\tfrac{\numepoch}{d_\hstep}\)$.  

Putting together the pieces, taking a union bound over the two covers
yields that, for each fixed $\hstep \in [\horizon]$, we have
\begin{align*}
\| e_\hstep^\eta(Q,\pi)\|_{\Sigma^{-1}_\hstep} & \leq d_\hstep \log
\big( 1 + \tfrac{\numepoch}{d_\hstep \lambda} \big) + \log
\Ncover_{\infty}(\epsq; \Q) + \log \Ncover_{\infty, 1}(\epspi; \Pi) +
\log \(\tfrac{1}{\delta}\) + 4 \sqrt{\numepoch} \epsq
\end{align*}
with probability at least $1-\delta$.  Finally, we take a union bound
over all $h \in [\horizon]$, which forces us to redefine $\delta$ to
$\frac{\delta}{\horizon}$ in the above bound.  This completes the
proof of the uniform bound~\eqref{EqnSupremum}.


\subsubsection{Proof of Lemma~\ref{LemCovering}}
\label{SecProofLemCovering}

Since $\|\phi(\state, \action)\|_2 \leq 1$, for any pair of weight
vectors $w, w' \in \R^d$, we have $\sup_{(\state, \action)}
|\inprod{\phi(\state, \action)}{w - w'}\|_2 \leq \|w - w'\|_2$. Thus,
the bound~\eqref{EqnQcover} follows from standard results on coverings
of Euclidean balls (cf. Example 5.8 in the
book~\cite{wainwright2019high}).

As for the bound~\eqref{EqnPiCover}, we claim that
\begin{align}
\label{EqnNearby}  
  \sum_{\action \in \ActionSpace} \big| \pi_{\theta'}(\action \mid
  \state) - \pi_{\theta}(\action \mid \state)| & \leq 8 \| \theta -
  \theta'\|_2, \qquad \mbox{for all $\state \in \StateSpace$.}
  \end{align}
Taking this claim as given for the moment, it suffices to obtain an
$\epsilon/8$-cover of the ball $\B(R)$ in the $\ell_2$-norm, and
applying the same standard results yields the claimed
bound~\eqref{EqnPiCover}.

\noindent It remains to prove the claim~\eqref{EqnNearby}.


\paragraph{Proof of the claim~\eqref{EqnNearby}:}

Let us state and prove the claim~\eqref{EqnNearby} more formally as a
lemma.  It applies to the softmax policy \mbox{$\pi_{\theta}(\action
  \mid \state) = \frac{\exp \{ \inprod{\phi(\state, \action)}{\theta}
    \}}{ \sum_{\action' \in \ActionSpace} \exp(\inprod{\phi(\state,
      \action')}{ \theta} )}$.}

\begin{lemma}[Nearby Policies]
\label{lem:NearbyPolicies}
Consider a feature mapping $\phi: \StateSpace \times \ActionSpace
\rightarrow \R^d$ such that \mbox{$\|\phi(\state, \action)\|_2 \leq
  1$} uniformly for all pairs $(\state, \action)$.  Then \mbox{for all
  $\state \in \StateSpace$,} we have
\begin{align}
  \sum_{\action \in \ActionSpace} \big| \pi_{\theta'}(\action \mid
  \state) - \pi_{\theta}(\action \mid \state)| & \leq 8 \| \theta -
  \theta'\|_2, 
\end{align}
valid for any pair $\theta, \theta' \in \R^d$ such that $\|\theta -
\theta' \|_2 \leq \tfrac{1}{2}$.
\end{lemma}
\begin{proof}
Dividing $\pi_{\theta'}(\state, \action)$ by $\pi_\theta(\state,
\action)$ yields
\begin{align*}
T \defeq \frac{\pi_{\theta'}(\action \mid \state)}{\pi_\theta(\action
  \mid \state)} & = \frac{e^{ \inprod{\phi(\state,
      \action)}{\theta'}}}{e^{\inprod{\phi(\state, \action)}{\theta}}}
\times \frac{\sum_{a''} e^{
    \inprod{\phi(s,a'')}{\theta}}}{\sum_{\atil}e^{
    \inprod{\phi(s,\atil)}{\theta'}}} \\
& = e^{\inprod{\phi(\state, \action)}{\theta'-\theta}} \times
\sum_{a''} \( e^{ \inprod{\phi(s,a'')}{\theta-\theta'} } \times
\frac{e^{ \inprod{\phi(s,a'')}{\theta'}} }{\sum_{\atil} e^{
    \inprod{\phi(s,\atil)}{\theta'}}}\) \\
& = e^{ \inprod{\phi(\state, \action)}{\theta'-\theta}} \times
\sum_{a''} \pi_{\theta'}(a'' \mid s) e^{ \inprod{\phi(s,a'')}{
    \theta-\theta'}}.
\end{align*}
By Cauchy-Schwarz and the assumption on $\phi$, we have the bound
$|\inprod{\theta(\state, \action)}{\gamma}| \leq \|\gamma\|_2$, valid
for any vector $\gamma$.  Monotonicity of the exponential allows us to
exponentiate this inequality.  Combined with the fact that
$\pi_{\theta'}(a'' \mid s) \geq 0$, we find that
\begin{align}
\label{EqnEspresso}
T \leq e^{\| \theta'-\theta\|_2} \; \sum_{a'' \in \ActionSpace}
\pi_{\theta'}(a'' \mid s) e^{\|\theta-\theta'\|_2} \stackrel{(i)}{=}
e^{2\|\theta - \theta'\|_2} \stackrel{(ii)}{\leq} 1 +
4\|\theta-\theta' \|_2,
\end{align}
where step (i) uses the fact that $\pi_\theta$ is a probability
distribution over the action space; and step (ii) follows by combining
the elementary inequality $e^x \leq 1+2 x$, valid for all $x \in
[0,1]$, with our assumption that $\|\theta - \theta' \|_2 \leq 1/2$.

Recalling that $T = \tfrac{\pi_{\theta'}(\action \mid
  \state)}{\pi_\theta(\action \mid \state)}$, re-arranging the
inequality~\eqref{EqnEspresso} yields the bound
\begin{align*}
\pi_{\theta'}(\action \mid \state) - \pi_\theta(\action \mid \state)
\leq 4 \pi_\theta(\action \mid \state) \; \|\theta-\theta' \|_2,
\end{align*}
valid uniformly over all pairs $(\state, \action)$. We can apply the
same argument with the roles of $\theta$ and $\theta'$ reversed, and
combining the two bounds yields
\begin{align*}
|\pi_{\theta'}(\action \mid \state) - \pi_\theta(\action \mid \state)|
& \leq 4 \|\theta-\theta'\|_2 \max\{ \pi_{\theta}(\action \mid
\state), \; \pi_{\theta'}(\action \mid \state)\},
\end{align*}
again uniformly over all pairs $(\state, \action)$.  Now summing over
the actions $\action$, we find that
\begin{align*}
\sum_{\action \in \ActionSpace} \big| \pi_{\theta'}(\action \mid
\state) - \pi_{\theta}(\action \mid \state) \big| & \leq 4
\sum_{\action \in \ActionSpace} \max \big \{ \pi_{\theta}(\action \mid
\state), \pi_{\theta'}(\action \mid \state) \big \} \; \|\theta -
\theta'\|_2 \\
& \leq 4 \sum_{\action \in \ActionSpace} \big \{ \pi_{\theta}(\action
\mid \state) + \pi_{\theta'}(\action \mid \state) \big \} \|\theta -
\theta'\|_2 \\
& = 8 \| \theta -\theta'\|_2,
\end{align*}
where the last step uses the fact that $\pi_{\theta}$ and
$\pi_{\theta'}$ are probability distributions over the action space.
Note that this inequality holds for all states $\state$, as claimed.
\end{proof}


\section{Actor's analysis: Proof of Proposition~\ref{PropActorMain}}
\label{AppPropActor}

In order to prove this claim, we require an auxiliary result that
re-expresses the mirror update rule.  Given the $Q$-value function
\mbox{$Q(\state, \action) \defeq \inprod{\phi(\state, \action)}{w}$,}
consider the linear update $\theta^+ \defeq \theta + \eta w$, and the
induced soft-max policy $\pi_{\theta^+}$.  The following auxiliary
result extracts a useful property of this update:
\begin{lemma}[Update in Natural Policy Gradient]
\label{LemPolicyUpdate}
For any function $F: \StateSpace \rightarrow \R$, we have
\begin{align}
Q(\state, \action) - F(\state) = \frac{1}{\eta} \Bigg[\log
  \frac{\pi_{\theta^+}(\state, \action)}{\pi_{\theta}(\state,
    \action)} + \log \(\sum_{\action' \in \ActionSpace}
  \pi_\theta(\state, \action') e^{\eta \(Q(\state, \action') -
    F(\state)\)}\) \Bigg],
\end{align}
valid for all pairs $(\state, \action)$.
\end{lemma}
\noindent See Section~\ref{SecProofLemPolicyUpdate} for the proof of
this claim. \\

\noindent Turning to the proof of the proposition, we have
\begin{align}
\label{EqnTwoTerms}
  V^{\pi}_{1, \MDP_\actorit}(\state_1) -
  V^{\pi_\actorit}_{1,M_\actorit}(\state_1) & \stackrel{(i)}{=} \sumh
  \E_{(\MyState_\hstep, \Action_\hstep) \sim \pi}
  \Big[\Advan_{h,M_\actorit}^{\pi_\actorit}(\MyState_\hstep,
    \Action_\hstep) \Big] \stackrel{(ii)}{=} \frac{1}{\eta} \sumh
  X_{\hstep, \actorit},
\end{align}
where we have introduced the shorthand
\begin{align}
\label{EqnXterm}
  X_{\hstep, \actorit} & \defeq \E_{(\MyState_\hstep, \Action_\hstep)
    \sim \pi} \Bigg[\log
    \frac{\pi_{\theta_{\actorit+1}}(\MyState_\hstep,
      \Action_\hstep)}{\pi_{\theta_\actorit}(\MyState_\hstep,
      \Action_\hstep)} + \log \(\E_{\Action_\hstep' \sim
      \pi_\actorit(\cdot \mid \MyState_\hstep)} \Big[e^{\eta
        \Advan^{\pi_\actorit}_{h,M_\actorit}(\MyState_\hstep,
        \Action_\hstep')} \Big] \) \Bigg].
\end{align}
Here step (i) follows from the simulation lemma
(e.g.,~\cite{kakade2003sample}), and step (ii) makes use of
Lemma~\ref{LemPolicyUpdate} with $F(\state) =
V^{\pi_\actorit}_{\hstep, \MDP_\actorit}(\state)$, along with the
definition of the advantage function---namely,
$\Advan^{\pi_\actorit}_{h,M_\actorit}(\state, \action) =
Q^{\pi_\actorit}_{h,M_\actorit}(\state, \action) -
V^{\pi_\actorit}_{h,M_\actorit}(\state)$.

For each $\hstep \in [\horizon]$ and $\actorit \in [\actortot]$, we
now bound the two terms within the definition~\eqref{EqnXterm} of
$X_{\hstep, \actorit}$ separately.  In particular, we derive a
telescoping relationship for the first term, and a uniform bound on
the second term.

\paragraph{First term:}
For any pair of policies $\pi, \pitil$ and $\state$, we introduce the
shorthand
\begin{align*}
  \KLTERM{\pi}{\pitil}{\state} & \defeq KL \(\pi(\cdot \mid
  s) \| \pitil (\cdot \mid \state) \).
\end{align*}
From the definition of KL divergence, for each $\state_\hstep$, we
have
\begin{align}
\sum_{a_\hstep \in \ActionSpace } \pi(a_\hstep \mid s_\hstep) \log
\frac{\pi_{\actorit + 1}(s_\hstep, a_\hstep)}{\pi_\actorit(s_\hstep,
  a_\hstep)} & = \sum_{a_\hstep } \pi(a_\hstep \mid s_\hstep)
\Big[\log \frac{\pi_{\actorit + 1}(s_\hstep,
    a_\hstep)}{\pi(s_\hstep,a_\hstep)} - \log
  \frac{\pi_\actorit(s_\hstep, a_\hstep)}{\pi(s_\hstep, a_\hstep)}
  \Big] \notag \\
\label{EqnTermOneBound}
& = -\KLTERM{\pi}{\pi_{\actorit + 1}}{\state_\hstep} +
\KLTERM{\pi}{\pi_{\actorit}}{\state_\hstep}.
\end{align}

\paragraph{Second term:}
We begin with the elementary inequality $e^x \leq 1 + x + x^2$ valid
for all $x \in [0,1]$.  By assumption, we have $|\eta \Advan_{h,
  \MDPit{\actorit}}^{\pi_\actorit}(\state, \action) | \leq 2\eta
 \leq 1$ for any pair $(\state, \action)$, and hence
\begin{align*}
e^{\eta \Advan_{h, \MDPit{\actorit}}^{\pi_\actorit}(\state, \action) }
& \leq 1 + \Big(\eta \Advan_{h,
  \MDPit{\actorit}}^{\pi_\actorit}(\state, \action) \Big) + \Big(\eta
\Advan_{h, \MDPit{\actorit}}^{\pi_\actorit}(\state, \action) \Big)^2
\; \leq 1 + \Big(\eta \Advan_{h,
  \MDPit{\actorit}}^{\pi_\actorit}(\state, \action) \Big) + 4\eta^2.
\end{align*}
By definition of the advantage function, we have $\E_{\Action_\hstep'
  \sim \pi_\actorit} \left[
  \Advan_{h,M_\actorit}^{\pi_\actorit}(s_\hstep, \Action_\hstep')
  \right] = 0$, so that we have
\begin{align}
\label{EqnTermTwoBound}  
  \log \left(\E_{\Action'_\hstep \sim \pi_\actorit} e^{\eta \Advan_{h,
      \MDPit{\actorit}}^{\pi_\actorit}(\state_\hstep, \Action'_\hstep)}
  \right) & \leq \log \left(1 + 4\eta^2  \right) \; \leq \;
  4\eta^2.
\end{align}

\paragraph{Combining the pieces:}  
Combining the bounds~\eqref{EqnTermOneBound}
and~\eqref{EqnTermTwoBound} yields
\begin{align*}
\frac{1}{\eta} X_{\hstep, \actorit} & \leq \frac{1}{\eta}
\E_{(\MyState_\hstep) \sim \pi}
\left[ -\KLTERM{\pi}{\pi_{\actorit +
      1}}{\MyState_\hstep} +
  \KLTERM{\pi}{\pi_{\actorit}}{\MyState_\hstep} \right] + 4\eta.
\end{align*}
Averaging this bound over all $\actorit \in [\actortot]$ and
exploiting the telescoping of the terms yields
\begin{align*}
\frac{1}{\eta \actortot} \sumactor X_{\hstep, \actorit} & \leq
\frac{1}{\eta \actortot} \E_{\MyState_\hstep \sim \pi} \left[
  -\KLTERM{\pi}{\pi_{\actorit + 1}}{\MyState_\hstep} +
  \KLTERM{\pi}{\pi_{1}}{\MyState_\hstep} \right] + 4\eta \\
& \stackrel{(i)}{\leq} \frac{1}{\eta \actortot}
\E_{(\MyState_\hstep) \sim \pi}
\KLTERM{\pi}{\pi_{1}}{\MyState_\hstep} + 4\eta \\
& \stackrel{(ii)}{\leq} \frac{1}{\eta \actortot} \log(|\ActionSpace|) +
4\eta,
\end{align*}
where step (i) follows by non-negativity of the KL divergence; and
step (ii) uses the fact that the KL divergence is at most
$\log(|\ActionSpace|)$.  Summing these bounds over $\hstep \in
[\horizon]$ yields
\begin{align*}
\frac{1}{\actortot} \sumactor \left \{ V^{\pi}_{1,
  \MDP_\actorit}(\state_1) - V^{\pi_\actorit}_{1,M_\actorit}(\state_1)
\right \} & = \frac{1}{\eta \actortot} \sumactor \sumh X_{\hstep,
  \actorit} \; \leq \; \horizon \left \{ \frac{1}{\eta \actortot}
\log(|\ActionSpace|) + 4\eta \right \},
\end{align*}
thereby establishing the claim~\eqref{EqnActorBoundGeneral}.
    
Finally, the bound~\eqref{EqnActorBound} follows by making the
particular stepsize choice $\eta = \sqrt{\frac{\log
    |\ActionSpace|}{\actortot}}$.  Note that the assumed lower bound
$\actortot \geq \log |\ActionSpace|$ ensures that $\eta \leq 1$, as
required to apply the bound~\eqref{EqnActorBoundGeneral}.

\subsection{Proof of Lemma~\ref{LemPolicyUpdate}}
\label{SecProofLemPolicyUpdate}
By definition of the soft-max policy, we have $\pi_{\theta^+}(\state,
\action) = \frac{\exp(\inprod{\phi(\state,
    \action)}{\theta^+})}{\sum_{\action' \in \ActionSpace} e^{
    \inprod{\phi(\state, \action')}{\theta^+}}}$.  Since $\theta_+ =
\theta + \eta w$, we can write
\begin{align*}
  \pi_{\theta^+}(\state, \action) = \frac{e^{ \inprod{\phi(\state,
        \action)}{\theta + \eta w}}}{\sum_{\action' \in \ActionSpace}
    e^{ \inprod{\phi(\state, \action')}{ \theta + \eta w }}}
& = \frac{e^{ \inprod{\phi(\state, \action)}{\theta} } e^{\eta
    \inprod{\phi(\state, \action)}{w} } }{\sum_{\action'
    \in \ActionSpace} e^{ \inprod{\phi(\state, \action')}{\theta}}
    e^{\eta \inprod{\phi(\state, \action')}{w} }} \\
& = \frac{e^{ \inprod{\phi(\state, \action)}{\theta}}}{\sum_{\atil
      \in \ActionSpace} e^{ \inprod{\phi(\state, \atil)}{\theta}}}
  \times \frac{e^{\eta \inprod{\phi(\state, \action)}{w} }
  }{\sum_{\action' \in
    \ActionSpace} \frac{e^{ \inprod{\phi(\state, \action')}{\theta}
  }}{\sum_{\atil \in \ActionSpace} e^{\inprod{\phi(\state,
        \atil)}{\theta}}} e^{\eta \inprod{\phi(\state, \action')}{w}
}} \\
& = \pi_\theta(\state, \action)\times \frac{e^{\eta
      \inprod{\phi(\state, \action)}{w} } }{\sum_{\action'
      \in \ActionSpace} \pi_\theta(\state, \action') e^{\eta
      \inprod{\phi(\state, \action')}{w} }} \\
& = \pi_\theta(\state, \action) \times \frac{e^{\eta Q(\state,
      \action) } }{\sum_{a' \in \ActionSpace} \pi_\theta(\state,
    \action') e^{\eta Q(\state, \action')}} 
\end{align*}
where the last step uses the definition of $Q$.
Multiplying both sides by $e^{-F(\state)}$ and re-arranging yields
\begin{align*}
\frac{\pi_{\theta^+}(\state, \action)}{\pi_{\theta}(\state, \action)}
\; \sum_{\action' \in \ActionSpace} \pi_\theta(\state, \action')
e^{\eta [Q(\state, \action') - F(\state)]} = e^{\eta [Q(\state,
    \action) - F(\state)] },
\end{align*}
which is equivalent to the claim.


\section{Proof of Theorem~\ref{thm:LowerBound}}

We now turn to the proof of the lower bound stated in
Theorem~\ref{thm:LowerBound}.  In Section~\ref{SecMDPClass}, we
describe the class of MDPs used in the construction, along with the
data generating procedure.  Section~\ref{SecMainArgLower} provides the
core argument, which involves three auxiliary lemmas.  These lemmas
are proved in Sections~\ref{SecProofReductionLower},
~\ref{SecProofTestingLower} and~\ref{SecProofUncertainUpper},
respectively.


\subsection{MDP class and data collection}
\label{SecMDPClass}

For a given horizon $\horizon$ and dimension $d$, we define a family
of MDPs that are parameterized by a Boolean vector $u = (u_1, \ldots,
u_\horizon) \in \{-1, +1 \}^{d \horizon}$, where each $u_\hstep \in
\{-1, +1 \}^d$.  For a given Boolean vector $u$, the associated MDP
$\MDP_u$ has the following structure:
\begin{description}
\item[State space and transition:] At each time step $\hstep$, there
  is only one state---viz. $\StateSpace = \{ \state \}$.  Since there
  is a single state, the transition is deterministic into the same
  state.
\item[Action space:] At each time step $\hstep$, the action space is
  given by $\ActionSpace = \{-1, 0, +1\}^d$.
\item[Feature map:] At each time step $\hstep$, the feature map
  $\phi: \StateSpace \times \ActionSpace \rightarrow \R^{d+1}$ takes
  the form
\begin{align}
  \label{eqn:phi_lb}
  \phi(\state, \action) = \left [\tfrac{\action}{\sqrt{2 d}},
    \tfrac{1}{\sqrt{2}} \right].
\end{align} 
Notice that by construction, we have the bound \mbox{$\| \phi(\state,
  \action) \|_2 = \sqrt{\frac{\| \action \|^2_2}{2 d} + \frac{1}{2} }
  \leq 1$} for any state-action pair.
\item[Reward mean:] The mean reward at time step $\hstep$ is
  proportional to the inner product $\inprod{\action}{u_\hstep}$,
  where $u_\hstep \in \{-1, 1\}^d$ is the sub-vector associated with
  time step $\hstep$.  More precisely, we have
  \begin{align}
r_\hstep(\state, \action) & = \inprod{\phi_\hstep(\state,
  \action)}{\begin{bmatrix} \delta u_\hstep & 0 
\end{bmatrix}} \; = \; \tfrac{\delta}{2 \sqrt{d}}
\inprod{a}{u_\hstep},
  \end{align}
  where $\delta > 0$ is a parameter to be specified in the
  proof.
\item[Low-rank MDP model:] It is easy to verify that the MDP so defined is low-rank; here we only verify explicitly the regularity conditions about the size of the radii so that the setting for the lower bound matches the setting that \textsc{Pacle} can handle. We need to verify explicitly that we can represent the action value function for any policy $\pi$, namely that there exists $w_h^{\pi}$ such that the action value function $Q_h^\pi(\state,\action) = \inprod{\phi_h(\state,\action)}{w_h^\pi}$ with $\| w_h^\pi \|_2 \leq (\horizon - \hstep+1)/(2\horizon)$. One can verify that for any policy $\pi$ we have 
\begin{align}
	w_h^\pi = [\delta u_h, \sqrt{2}V_{\hstep+1}^\pi ], \qquad \forall h \in [H].
\end{align}
A sufficient condition for the regularity conditions to be satisfied is when
\begin{align}
\label{eqn:deltabound}
	\delta \| u_h \|_2 \leq 1/(2\horizon) \rightarrow \delta \leq \frac{1}{2\sqrt{d}\horizon},
\end{align} which implies $|V_h^\pi| \leq (\horizon-\hstep+1)/(2\horizon)$ and hence
\begin{align}
	\| w_h^\pi \|_2 \leq  \delta_2 \|u_h\|_2 +  \sqrt{2}|V_{\hstep+1}^\pi| \leq  (\horizon - \hstep+1)/(2\horizon), \qquad \forall h \in [H].
\end{align}
In \cref{LemTestingLower} we choose $\delta = \tfrac{d\sqrt{\horizon}}{\sqrt{2 \Nsamp}}$ which implies the lemma holds when
\begin{align}
	\frac{d\sqrt{\horizon}}{\sqrt{2 \Nsamp}} \leq \frac{1}{2\sqrt{d}\horizon} \rightarrow n \geq 2d^3\horizon^3.
	\end{align}

\item[Reward observations:] We observe the mean reward contaminated by
  additive Gaussian noise, so that the reward distribution has the
  form
\begin{align}
\label{eqn:rew_lb}
R_\hstep(\state, \action) \sim \N \left( \frac{\delta}{\sqrt{2 d}}
\inprod{\action}{u_\hstep}, 1 \right).
\end{align}
\end{description}

\paragraph{Data collection:}  We assume that
the $\Nsamp$ samples are collected according to the following
non-adaptive process.
\begin{itemize}
\item Each time step $\hstep \in [\horizon]$ is allocated $\Nsub
  \defeq \Nsamp/\horizon$ samples (assumed to be an integer for
  simplicity).
\item For each $\hstep$, the dataset $\DataSet_\hstep$ is generated by
  playing each action $\action \in \{e_1, \ldots, e_d, \zerovec \}$
  exactly $\Nsub/(d + 1)$ times, where $e_j \in \{0,1\}^d$
  denotes the standard basis vector with a single one in index $j$.
\end{itemize}


\subsection{Main argument}
\label{SecMainArgLower}
With this set-up, we now introduce the three lemmas that form the core
of the proof.  For any given $u \in \{-1, +1\}^{d \horizon}$, let
$\Qprob_u$ denote the distribution of the data $\DataSet$ when the
sampling process is applied to the MDP $\MDP_u$, and let $\E_u$ denote
expectations under this distribution.  Our first lemma exploits the
Assouad construction so as to reduce the problem of finding a good
policy to a family of testing problems.
\begin{lemma}[Reduction to testing]
\label{LemReductionLower}
For any estimated policy $\pialg$, we have
\begin{align}
\label{EqnReductionLower}  
\sup_{u \in \mathcal U} \E_u [ \Vstar_u - V^{\pialg}_u] & \geq
\frac{\delta}{\sqrt{2d} } \frac{d \horizon}{2}
\min_{ \substack{u,u'
    \in \Ucal \\ \HammDist{u}{u'} = 1}} \inf_{\psi} \Big[
  \Qprob_u(\psi(\DataSet) \neq u) + \Qprob_{u'}(\psi(\DataSet) \neq
  u') \Big],
\end{align}
where a test function $\psi$ is a measurable function of the data
taking values in $\{u, u'\}$.
\end{lemma}
\noindent See Section~\ref{SecProofReductionLower} for the proof.\\

Our second lemma involves further lower bounding the testing error in
the bound~\eqref{EqnTestingLower}.  In particular, we prove the
following:
\begin{lemma}[Lower bound on testing error]
  \label{LemTestingLower}
  For the given family of distributions $\{\Qprob_u, u \in \Ucal\}$,
  we have
  \begin{align}
\label{EqnTestingLower}    
\min_{ \substack{u,u' \in \Ucal \\ \HammDist{u}{u'} = 1}} \inf_{\psi}
\Big[ \Qprob_u(\psi(\DataSet) \neq u) + \Qprob_{u'}(\psi(\DataSet)
  \neq u') \Big] & \geq \(1 - \sqrt{\frac{1}{2} \frac{\Nsub
    \delta^2}{d^2} } \).
  \end{align}
Thus, the testing error is lower bounded by $\tfrac{1}{2}$ with the
choice $\delta = \tfrac{d}{\sqrt{2 \Nsub}}$.
\end{lemma}
\noindent See Section~\ref{SecProofTestingLower} for the proof.\\

Combining the claims of Lemmas~\ref{LemReductionLower}
and~\ref{LemTestingLower}, along with the choice $\delta =
\tfrac{d}{\sqrt{2 \Nsub}}$, yields the lower bound
\begin{align}
\label{EqnInterLower}  
\sup_{u \in \mathcal U} \E_u [ \Vstar_u - V^{\pialg}_u] & \geq
\frac{\delta}{\sqrt{2d} } \frac{d \horizon}{2} \frac{1}{2} \; \geq
\frac{1}{8} d \horizon \sqrt{\frac{d}{\Nsub}}.
\end{align}

Thus, the only remaining step is to relate this lower bound to the
uncertainty function $\Uncertain(\pi; \sqrt{d})$ associated with our
family of MDPs. More precisely, we prove the following:
\begin{lemma}
\label{LemUncertainUpper}  
There is a universal constant such that
\begin{align}  
\sup_{\pi} \Uncertain(\pi; \sqrt{d}) & \leq c d \horizon \sqrt{
  \frac{d}{\Nsub}}
\end{align}
\end{lemma}
\noindent See Section~\ref{SecProofUncertainUpper} for the proof. \\

\noindent Combining Lemma~\ref{LemUncertainUpper} with the lower
bound~\eqref{EqnInterLower} concludes the proof of the theorem.\\

\vspace*{0.1in}

\noindent It remains to prove our auxiliary lemmas, and we do so in the
following subsections.


\subsection{Proof of Lemma~\ref{LemReductionLower}}
\label{SecProofReductionLower}

For a given $u \in \Ucal$, let $\pistar_u$ be the optimal policy on
$M_u$ and let $\Vstar_u$ the optimal value function.  For any
estimated policy $\pi$, we define the estimated sign vector $u^\pi \in
\{-1, 1\}^{d \horizon}$ with entries $[u^\pi]_{hi} \defeq
\sign(\E_{\action \sim \pi_h} a_i)$.

\noindent With this set-up, we prove the lemma in two steps:
\begin{enumerate}
\item[(a)] First, we show that the value function gap $\Vstar_u -
  V^\pi_u$ can be lower bounded in terms of the Hamming distance
  \begin{align}
\label{EqnValue2Hamming}    
     \Vstar_u - V^\pi_u & \geq \frac{\delta}{\sqrt{2 d }}
     \HammDist{u^\pi}{u}.
  \end{align}
\item[(b)] We use Assaoud's method to lower bound the estimation
  error in the Hamming distance.
\end{enumerate}

\paragraph{Step (a):}
Since the optimal action at timestep $h$ on $M_u$ is $u_\hstep$, by
inspection, the associated suboptimality of $\pi$ on $M_u$ compared to
the optimal policy on $M_u$ is
\begin{align*}
  \Vstar_u - V^\pi_u & = \frac{1}{\sqrt{2 d}} \sumh \Big[
    \inprod{u_\hstep}{\delta u_\hstep} - \E_{\action \sim \pi_\hstep}
    \inprod{\action}{\delta u_\hstep} \Big] \\
& = \frac{\delta}{\sqrt{2 d}}
\sumh \sum_{i=1}^{d} \Big[[u]_{hi}
  [u]_{hi} - [\E_{a \sim \pi_\hstep} a]_{i} [u]_{hi} \Big] \\
& = \frac{\delta}{\sqrt{2 d}} \sumh \sum_{i=1}^{d} \Big( [u]_{hi} -
      [\E_{a \sim \pi_\hstep} a]_i\Big) [u]_{hi} \\
& = \frac{\delta}{\sqrt{2 d}} \sumh \sum_{i=1}^{d} \Big| [u]_{hi} -
      [\E_{a \sim \pi_\hstep} a]_i\Big|.
\end{align*}
Now recalling that $[u^\pi]_{hi} \defeq \sign(\E_{\action \sim \pi_h}
a_i)$, we have the lower bound
\begin{align*}
  \Vstar_u - V^\pi_u & \geq \frac{\delta}{\sqrt{2 d}} \sumh
  \sum_{i=1}^{d} \Big| [u]_{hi} - [\E_{a \sim \pi_\hstep} a]_i\Big|
  \1\{ u^\pi_{hi} \neq [u]_{hi} \} \\
    & \geq \frac{\delta}{\sqrt{2 d }} \sumh \sum_{i=1}^{d} \1 \{
      u^\pi_{hi} \neq [u]_{hi} \} \\
    & = \frac{\delta}{\sqrt{2 d}} \HammDist{u^\pi}{u},
\end{align*}
which establishes the lower bound~\eqref{EqnValue2Hamming}.

\paragraph{Step (b):}  We can now apply
Assouad's method (cf. Lemma 2.12 in the book~\cite{Tsybakov09}), so as
to conclude that for any estimated policy $\pi$, we have
\begin{align}
\sup_{u \in \Ucal} \E_u \big [\HammDist{u^\pi}{u} \big] & \geq \frac{d
  \horizon}{2} \min_{u,u' \mid \HammDist{u}{u'} = 1} \inf_{\psi} \Big[
  \Pro_u(\psi \neq u) + \Pro_{u'}(\psi \neq u') \Big]
\end{align}
where $\inf_{\psi}$ denotes the minimum over all test functions taking
values in $\{u, u' \}$.


\subsection{Proof of Lemma~\ref{LemTestingLower}}
\label{SecProofTestingLower}

We begin by observing that the testing error can be lower bounded in
terms of the KL divergence as
\begin{align}
\label{EqnTestingKL}
\min_{ \substack{ u, u' \in \Ucal \\ \HammDist{u}{u'} = 1}}
\inf_{\psi} \Big[ \Pro_u(\psi \neq u) + \Pro_{u'}(\psi \neq u') \Big]
& \geq 1 - \Big(\tfrac{1}{2} \max_{ \substack{ u, u' \in \Ucal
    \\ \HammDist{u}{u'} = 1}} \KLDist{\Qprob_u}{\Qprob_{u'}}
\Big)^{1/2}.
\end{align}
For instance, see Theorem 2.12 in \cite{Tsybakov09}.

Thus, in order to prove Lemma~\ref{LemTestingLower}, it remains to
bound the Kullback-Leibler divergence of the distributions $\Qprob_u$
and $\Qprob_{u'}$ for pairs $u, u' \in \{-1, +1 \}^{d \horizon}$ that
differ only in a single coordinate.

By construction, the only stochasticity in the dataset lies in the
rewards.  For any given $u$, equation~\eqref{eqn:rew_lb} implies that
the distribution over rewards has the product form
\begin{align*}
\Qprob_{u} & = \prod_{h=1}^\horizon \prod_{i=1}^{d}
\prod_{j=1}^{\frac{\Nsamp_{h}}{d}} \N \(
\frac{e_i^\top}{\sqrt{2d_\hstep}} (\delta u_\hstep), 1 \).
\end{align*}
Notice that each normal distribution in the above display for
$\Qprob_u$ is identical to the corresponding factor in $\Qprob_{u'}$
except for the single index in which the vectors $u$ and $u'$ differ.
Thus, applying the chain rule for KL divergence yields
\begin{align*}
\KLDist{\Qprob_{u}}{\Qprob_{u'}} = \sum_{k=1}^{\frac{\Nsub}{d}}
\KLDist{\N\( \frac{\delta}{\sqrt{2d}},1\)}{\N\(
  \frac{-\delta}{\sqrt{2d}},1\)} 
& = \frac{\Nsub}{2d} \( 2\frac{\delta}{\sqrt{2d}} \)^2 \\
& = \frac{\Nsub \delta^2}{d^2},
\end{align*}
valid for any pair $u, u'$ differing in a single coordinate.
Substituting back into the lower bound~\eqref{EqnTestingKL}
yields the claim.


\subsection{Proof of Lemma~\ref{LemUncertainUpper}}
\label{SecProofUncertainUpper}

Recall that by definition, we have \mbox{$\Uncertain(\pi ; \sqrt{d}) =
  \sqrt{d} \sumh \|\phi^{\pi}_\hstep \|_{\Sigma^{-1}_\hstep}$.}
Consequently, in order to establish the claim, it suffices to show
there is a universal constant $c$ such that
\begin{align}
\label{EqnIntermediateUncertainUpper}  
\sup_{\pi \in \Pi} \| \phi^{\pi}_\hstep \|_{\Sigma^{-1}_\hstep} & \leq
c \frac{d}{\sqrt{\Nsub}} \qquad \mbox{for each $\hstep \in
  [\horizon]$.}
\end{align}

Now denote with $[x]_{1:p}$ the first $p$ components of the vector
$x$, and with $[x]_{p}$ the $p$ component of $x$. Using the triangle inequality we can write
\begin{align*}
\|\phi^{\pi}_\hstep \|_{\Sigma^{-1}_\hstep} \leq \|\Big[
  [\phi^{\pi}_\hstep]_{1:d},0\Big] \|_{\Sigma^{-1}_\hstep} + \| \Big[
  0, [\phi^{\pi}_\hstep]_{d+1}\Big] \|_{\Sigma^{-1}_\hstep}.
\end{align*}
Next, we use a technical lemma to compute the inverse of
$\Sigma_\hstep$. By construction $\Sigma_\hstep$ is an arrowhead
matrix, i.e., can be written as
\begin{align*}
\Sigma_\hstep =
\begin{bmatrix}
	D & v \\ v^\top & b
\end{bmatrix}
\end{align*}
where we let the normalization constants inside of $\phi$ in
\cref{eqn:phi_lb} to be
\begin{align*}
	\gamma = \frac{1}{\sqrt{2d_\hstep}}, \qquad c =
        \frac{1}{\sqrt{2}}
\end{align*}
to define
 $D \in \R^{d\times d}$ as a diagonal matrix with entries
\begin{align*}
[D]_{ii} = \gamma^2\frac{\Nsub}{d} + \lambda
\end{align*}
and $v \in \R^d$ is a vector with entries
\begin{align*}
  [v]_i = \gamma c\frac{\Nsub}{d} 
\end{align*}
and $b \in \R$ is a scalar
\begin{align*}
  b = c^2\( \Nsub + \frac{\Nsub}{d}\)  + \lambda.
\end{align*}

The inverse of $\Sigma_\hstep$ can then be computed explicitly using
known formulas for block matrices or arrowhead matrices. We arrive to
\begin{align*}
\Sigma^{-1}_\hstep =
\begin{bmatrix}
	D' & v' \\ v'^\top & b'
\end{bmatrix}
\end{align*}
where we define the entries in a second. First, the inverse of the
Schur complement is
\begin{align*}
  b' \defeq (b - v^\top D^{-1} v)^{-1} & = \(c^2\( \Nsub +
  \frac{\Nsub}{d}\) + \lambda - \sum_{i=1}^d \frac{\(\gamma
    c\frac{\Nsub}{d} \)^2}{\gamma^2\frac{\Nsub}{d} +
    \lambda} \)^{-1}.
\end{align*}
Our goal is to show that this is positive, which helps in simplifying
the final expression. Notice that
\begin{align*}
  \sum_{i=1}^d \frac{\(\gamma c \frac{\Nsub}{d}
    \)^2}{\gamma^2\frac{\Nsub}{d} + \lambda} < \sum_{i=1}^d
  \frac{\(\gamma c\frac{\Nsub}{d} \)^2}{\gamma^2\frac{\Nsub}{d}
  } = dc^2 \frac{\Nsub}{d} = c^2\Nsub.
\end{align*}
Thus
\begin{align*}
  (b')^{-1} = \(c^2\( \Nsub + \frac{\Nsub}{d}\) + \lambda -
  \sum_{i=1}^d \frac{\(\gamma c\frac{\Nsub}{d}
    \)^2}{\gamma^2\frac{\Nsub}{d} + \lambda} \) >
  c^2\frac{\Nsub}{d} + \lambda >0.
\end{align*}
These facts imply that the inverse of the above quantity is bounded as
\begin{align*}
  b' < \frac{d}{c^2\Nsub + d\lambda} < \frac{d}{c^2\Nsub}.
\end{align*}

Continuing the construction of the inverse, we obtain
\begin{align*}
  D' = \underbrace{D^{-1}}_{\defeq D'_1} + \underbrace{D^{-1}
    vb'v^\top D^{-1}}_{\defeq D'_2}
\end{align*}
Noice that $D'_1$ is symmetric positive definite with positive
diagonal elements and $D'_2$ is also symmetric positive semidefinite:
\begin{align*}
0 \prec D'_1 & = D^{-1} = \(\gamma^2\frac{\Nsub}{d} +
\lambda\)^{-1} I \prec \frac{d}{\gamma^2 \Nsub} I \\ D'_2 & =
\underbrace{b'}_{\geq 0} \underbrace{D^{-1}v}_{y}\underbrace{v^\top
  D^{-1}}_{y^\top} = b'yy^\top \succcurlyeq 0.
\end{align*}

We now use the above block expressions for $\Sigma^{-1}_\hstep$ to
bound
\begin{align*}
  \| \phi^{\pi}_\hstep \|_{\Sigma_\hstep^{-1}} \leq \| \Big[
    [\phi^{\pi}_\hstep]_{1:d},0 \Big]\|_{\Sigma_\hstep^{-1}} + \|
  \Big[ \vec 0, [\phi^{\pi}_\hstep]_{d+1}
    \Big]\|_{\Sigma_\hstep^{-1}}.
\end{align*}
By construction, $ [\phi^{\pi}_\hstep]_{1:d}$ only interacts with the
$D'$ block in $\Sigma_\hstep^{-1}$; using this and
$$\|x\|^2_{D'} = x^\top (D'_1 + D'_2) x  \leq \| x \|_2\(\|D_1'\|_2 + \| D'_2 \|_2\) \|x\|_2$$  we can write
\begin{align*}
\| \Big[ [\phi^{\pi}_\hstep]_{1:d},0 \Big]\|_{\Sigma_\hstep^{-1}} = \|
   [\phi^{\pi}_\hstep]_{1:d}\|_{D'} \leq \|
   [\phi^{\pi}_\hstep]_{1:d}\|_{2}\sqrt{\|D'_1\|_2 + \|D'_2\|_2}
\end{align*}
Likewise,
\begin{align*}
\| \Big[0,[\phi^{\pi}_\hstep]_{d+1} \Big]\|_{\Sigma_\hstep^{-1}} = \|
   [\phi^{\pi}_\hstep]_{d+1}\|_{b'}.
\end{align*} 

We now bound all norms:
\begin{align*}
  \| [\phi^{\pi}_\hstep]_{1:d}\|_{2} & \leq
  \frac{\|\1\|_2}{\sqrt{2d}} \leq \frac{1}{\sqrt{2}} \\ \| D'_1
  \|_2 & = \| D^{-1} \|_2 \lesssim \frac{2d^2}{\Nsub} \\ \|
  D'_2 \|_2 & \leq b'\|D^{-1}\|_2 \|v\|_2 \|v\|_2 \|D^{-1}\|_2
  \lesssim
  \underbrace{\frac{d}{\Nsub}}_{b'}\underbrace{\(\gamma
    \frac{\Nsub}{d}\)^2 \|\1\|_2^2}_{\| v\|_2^2}
  \underbrace{\frac{d^4}{n^2_\hstep}}_{\|D^{-1}\|^2_2} \lesssim
  \frac{d^2}{\Nsub}
\end{align*}
Substituting back yields the bound
\begin{align*}
  \| \Big[ [\phi^{\pi}_\hstep]_{1:d},0 \Big]\|_{\Sigma_\hstep^{-1}}
  \lesssim \frac{d}{\sqrt{\Nsub}}
\end{align*}
Similarly, we have
\begin{align*}
\|[\phi^{\pi}_\hstep]_{d+1}\|_{b'} & = \sqrt{\frac{1}{\sqrt{2}}
  b'\frac{1}{\sqrt{2}}} \leq \sqrt{\frac{1}{2}\frac{d}{c^2\Nsub}}
\lesssim \frac{\sqrt{d}}{\sqrt{\Nsub}}.
\end{align*}
Putting together the pieces yields the
claim~\eqref{EqnIntermediateUncertainUpper}.


\end{document}